\documentclass{article}

\usepackage[main,final]{neurips_2025}

\usepackage[utf8]{inputenc} 
\usepackage[T1]{fontenc}    
\usepackage{hyperref}       
\usepackage{url}            
\usepackage{booktabs}       
\usepackage{amsfonts}       
\usepackage{nicefrac}       
\usepackage{microtype}      
\usepackage{xcolor}         

\usepackage{xspace}
\usepackage{pifont}
\newcommand{\ourmethod}{{\fontfamily{lmtt}\selectfont \textbf{AuroRA}}\xspace}
\definecolor{ccr}{RGB}{10,110,150}
\definecolor{new_ccr}{RGB}{0,191,255}
\definecolor{darksalmon}{rgb}{0.91, 0.59, 0.48}
\usepackage[table,xcdraw,dvipsnames]{xcolor}
\usepackage{fontawesome5}
\usepackage{booktabs}
\usepackage{makecell}
\usepackage{bbding}
\usepackage{multirow}
\newcommand{\blue}[1]{$_{\color{BlueGreen}\downarrow #1}$}
\newcommand{\red}[1]{$_{\color{RedOrange}\uparrow #1}$}
\newcommand{\darkred}[1]{$_{\color{red}\uparrow #1}$}
\newcommand{\darkblue}[1]{$_{\color{Blue}\downarrow #1}$}
\hypersetup{
    colorlinks=true,
    linkcolor=red,
    citecolor=new_ccr,
    filecolor=magenta,      
    urlcolor=magenta,
    }
\usepackage{wrapfig}
\usepackage{enumitem}
\usepackage{amsmath}
\usepackage{amssymb}
\usepackage{mathtools}
\usepackage{amsthm}
\theoremstyle{plain}
\newtheorem{theorem}{Theorem}[section]
\newtheorem{proposition}[theorem]{Proposition}
\newtheorem{lemma}[theorem]{Lemma}

\theoremstyle{definition}
\newtheorem{definition}[theorem]{Definition}
\newtheorem{assumption}[theorem]{Assumption}
\theoremstyle{remark}

\usepackage[ruled]{algorithm2e}
\SetKwInOut{Input}{Input}\SetKwInOut{Output}{Output}
\SetKwProg{Fn}{Function}{:}{}
\SetKwComment{Comment}{$\triangleright$\ }{}
\SetKwFunction{MergeWeights}{MergeWeights}

\title{\ourmethod: Breaking Low-Rank Bottleneck of LoRA with Nonlinear Mapping}

\author{%
  Haonan Dong$^{1}$, Wenhao Zhu$^{1}$, Guojie Song$^{\dag 1}$, Liang Wang$^{2}$ \\
  $^{1}$State Key Laboratory of General Artificial Intelligence, \\ 
  School of Intelligence Science and Technology, Peking University, \\ $^{2}$Alibaba Group, $^\dag$ Corresponding author \\
  \small {\faEnvelope} \texttt{hndong25@stu.pku.edu.cn, gjsong@pku.edu.cn}
}

\begin{document}

\maketitle

\begin{abstract}
Low-Rank Adaptation (LoRA) is a widely adopted parameter-efficient fine-tuning (PEFT) method validated across NLP and CV domains. However, LoRA faces an inherent low-rank bottleneck: narrowing its performance gap with full fine-tuning requires increasing the rank of its parameter matrix, resulting in significant parameter overhead. Recent linear LoRA variants have attempted to enhance expressiveness by introducing additional linear mappings; however, their composition remains inherently linear and fails to fundamentally improve LoRA’s representational capacity. To address this limitation, we propose \ourmethod, which incorporates an Adaptive Nonlinear Layer (ANL) between two linear projectors to capture \emph{fixed} and \emph{learnable} nonlinearities. This combination forms an {\fontfamily{lmtt}\selectfont \textbf{MLP-like structure}} with a compressed rank, enabling flexible and precise approximation of diverse target functions while theoretically guaranteeing lower approximation errors and bounded gradients. Extensive experiments on 22 datasets and 6 pretrained models demonstrate that \ourmethod: (\textbf{I}) not only matches or surpasses full fine-tuning performance with only $6.18\%\sim25\%$ of LoRA’s parameters but also (\textbf{II}) outperforms competitive PEFT methods by up to $10.88\%$ in both NLP and CV tasks, and \textbf{(III)} exhibits robust performance across various rank configurations.
\end{abstract}

\vspace{-0.5em}
\section{Introduction}
\label{sec:intro}

\begin{wrapfigure}{r}{0.55\textwidth}
\vspace{-1.5em}
 \centering
 \includegraphics[width=\linewidth]{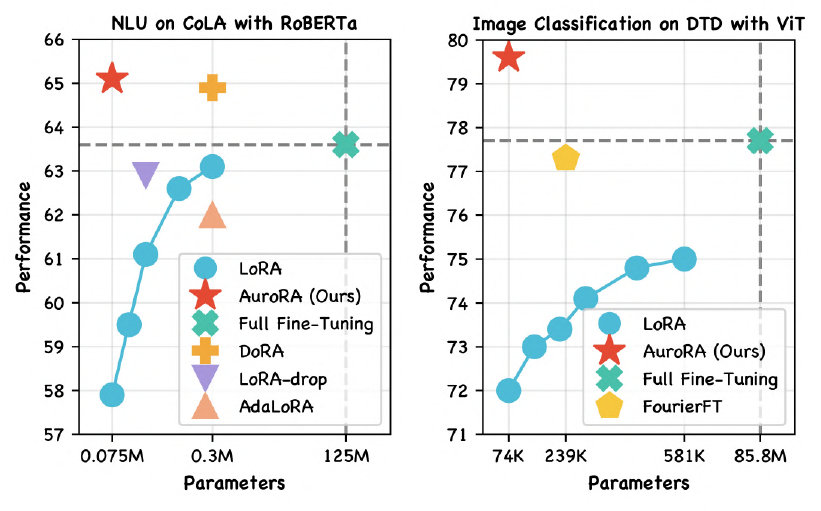}
  \vspace{-1.6em}
  \caption{The trade-off between parameters and performance of various fine-tuning methods on NLP (left) and CV (right) tasks. (\textit{\textbf{Left}}) In NLU, RoBERTa-Base is fine-tuned on \textsc{cola}, with LoRA ranks \( r = \{2, 3, 4, 6, 8\} \). (\textit{\textbf{Right}}) In image classification, ViT-Base is fine-tuned on \textsc{dtd}, with LoRA ranks \( r = \{2, 4, 6, 8, 12, 16\} \). \vspace{-1.4em}}
  \label{fig:intro-1}
\end{wrapfigure}
\vspace{-0.7em}
In recent years, pretrained models have demonstrated excellent generalization performance across numerous tasks in various domains \cite{roberta, deberta, llama3, debertav3, internlm2, tot}. In practical applications, to further unleash their powerful capabilities on specific downstream tasks, these models often require relevant fine-tuning \cite{alpaca,ft1,gpt3,ft2,meta-r1,gpv}. However, the increasing size of their parameters poses a significant challenge to fine-tuning all parameters \cite{adapter1,ft3}. To address this issue, the field of Parameter-Efficient Fine-Tuning (PEFT) has made substantial progress \cite{adapter1,gpt,lora,prefix,peft1,llm_adapters,bitfit}. The core idea is to fine-tune only a small subset of the model's parameters while freezing the majority of the pretrained parameters, achieving performance comparable to full fine-tuning \cite{peft_survey}.

\begin{wrapfigure}{r}{0.55\textwidth}
\vspace{-1.3em}
  \begin{center}
    \includegraphics[width=\linewidth]{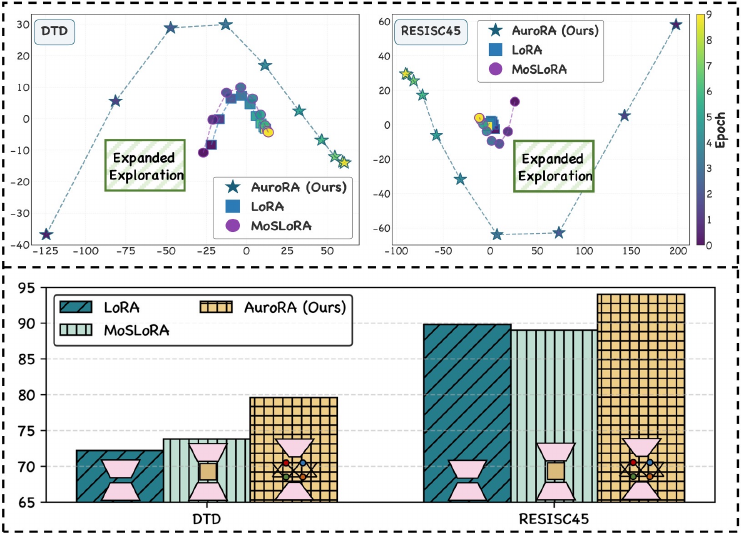}
  \end{center}
  \vspace{-1em}
  \caption{We evaluate LoRA, MoSLoRA, and our \ourmethod on \textsc{dtd} and \textsc{resisc45} datasets, employing ViT-Base with a rank of \(r=2\). (\textit{\textbf{Upper}}) We record the 
 \(\Delta W\) at the \(\{0,1,2,...,9\}\)-th epochs, and perform PCA visualization on these \(\Delta W\). We observe that \ourmethod is capable of exploring a broader parameter space. (\textit{\textbf{Lower}}) We present the accuracy results on both datasets.}
  \label{fig:intro-2}
  \vspace{-1.6em}
\end{wrapfigure}

LoRA is a commonly used state-of-the-art PEFT method \cite{lora}. Specifically, it assumes that the weight updates conform to a low-rank hypothesis and represents these updates using two low-rank matrices, i.e., \( W_0 + \Delta W = W_0 + BA \). Its performance has been validated in fields such as natural language processing (NLP) \cite{lora,lora_nlp2} and computer vision (CV) \cite{lora_cv1,lora_cv2}. Despite its significant success, LoRA still faces an inherent limitation, namely \textit{the low-rank bottleneck}, as illustrated in Figure \ref{fig:intro-1}. 
As the rank of LoRA increases, the model's performance improves, thereby narrowing the gap with full fine-tuning \cite{lora_rank}; however, the parameter cost grows proportionally with the rank, which weakens its parameter efficiency.
This dilemma leads us to the first research question: \ding{182} \textbf{\textit{Can we achieve a further balance between parameters and performance?}}

Recently, several linear LoRA variants have emerged \cite{adalora,salora,moslora,flora}. They introduce an \textit{additional matrix} between the \( B \) and \( A \) matrices of LoRA to weaken the correlation constraints between them, thereby enhancing LoRA's learning and expressive capabilities. 
Specifically, one approach involves the introduction of a \textit{diagonal matrix} to facilitate singular value decomposition \cite{adalora,salora}, while another approach incorporates an \textit{arbitrary matrix} to fuse subspaces \cite{moslora,flora}.
Nevertheless, LoRA's inherent linearity persists even when an additional matrix is introduced, preserving its fundamental structure as a linear mapping. As illustrated in Figure \ref{fig:intro-2}, when the rank is extremely low, MoSLoRA \cite{moslora} (a linear variant that incorporates an arbitrary matrix) has only a \textit{marginal} effect on expanding the exploration of \(\Delta W\), leading to a failure to further boost performance ($2.2\%\uparrow$ on \textsc{dtd} and $0.56\%\downarrow$ on \textsc{resisc45}). 
The structural characteristics and resultant performance limitations of linear variants naturally prompt our second research question: \ding{183} \textbf{\textit{Can we achieve more than marginal performance improvements by introducing a nonlinear transformation between LoRA's two linear layers?}}

Motivated by the above two research questions, this paper focuses on introducing nonlinear mappings into LoRA and further compressing the rank to achieve a better balance between parameters and performance. To this end, we propose a method called {\fontfamily{lmtt}\selectfont \textbf{\underline{A}ctivate Yo\underline{ur} L\underline{o}w-\underline{R}ank \underline{A}daptation}} (\ourmethod). We revisit LoRA through the lens of linear mappings and identify two critical limitations: (\textbf{I}) \textit{insufficient expressiveness} and (\textbf{II}) \textit{limited training flexibility}. To fully harness LoRA’s potential, \ourmethod introduces an \textbf{A}daptive \textbf{N}onlinear \textbf{L}ayer (\textbf{ANL}) between the low-rank matrices, forming an \textbf{MLP-like structure}. ANL employs a hybrid design of \textit{fixed} and \textit{learnable} nonlinearities to enhance model expressivity within a more compressed rank while enabling flexible training strategies to expand the explorable parameter space (Figures \ref{fig:intro-1} and \ref{fig:intro-2}). Theoretical analysis demonstrates that \ourmethod not only achieves a strictly lower approximation error than LoRA but also preserves bounded gradient norms. Experiments across NLP and CV tasks confirm the \underline{\textit{efficiency}}, \underline{\textit{generalizability}}, and \underline{\textit{robustness}} of \ourmethod. We further conduct ablation studies to dissect the contributions of fixed and learnable components, and evaluate its robustness against linear LoRA variants across multiple rank configurations. Our contributions can be summarized as follows:
\begin{itemize}[leftmargin=*]
\vspace{-0.7em}
\item[\ding{182}] \emph{\textbf{Perspective Shift.}} We systematically revisit two research lines of LoRA: \textit{the low-rank bottleneck} and \textit{linear LoRA variants}. By interpreting LoRA through the lens of linear mappings, we address both research questions within a unified framework, providing theoretical analyses.
\vspace{-0.6em}
\item[\ding{183}] \emph{\textbf{Nonlinear Proposal.}} We propose \ourmethod, which introduces nonlinear mappings into LoRA and further compresses the rank, resulting in a superior balance between parameters and performance, paving the way for further unlocking the significant potential of LoRA.
\vspace{-0.6em}
\item[\ding{184}] \emph{\textbf{Experimental Validation.}} Extensive experiments on 22 datasets and 6 pretrained models showcase that \ourmethod: \textbf{(I)} not only matches or surpasses full fine-tuning performance with only $6.18\%\sim25\%$ of LoRA’s parameters but also \textbf{(II)} outperforms competitive PEFT methods by up to $10.88\%$ in NLP and CV tasks, and \textbf{(III)} exhibits robust performance across various rank configurations. 
\end{itemize}

\vspace{-0.5em}
\section{Methodology}
\vspace{-0.5em}
As illustrated in Figure \ref{fig:overview}, we introduce \ourmethod, an extension of LoRA that incorporates nonlinear mappings to overcome the inherent low-rank bottleneck. We reinterpret LoRA as a \textbf{two-layer linear mapping}, whereas our proposed \ourmethod transforms it into an \textbf{MLP-like structure} by introducing an adaptive nonlinear layer.

\vspace{-0.5em}
\subsection{LoRA: A Two-Layer Linear Mapping}
\vspace{-0.5em}
In standard LoRA \cite{lora}, the weight update \( \Delta \mathcal{W} \) for a pre-trained weight matrix \( \mathcal{W}_0 \) is approximated as the product of two low-rank matrices:
\begin{equation}
    \Delta \mathcal{W} = \mathbf{B} \mathbf{A},
    \end{equation}
where \( \mathbf{A} \in \mathbb{R}^{r \times d_{\text{in}}} \) and \( \mathbf{B} \in \mathbb{R}^{d_{\text{out}} \times r} \), with the rank \( r \) satisfying \( r \ll \min(d_{\text{in}}, d_{\text{out}}) \). The forward propagation for an input vector \( \mathbf{x} \in \mathbb{R}^{d_{\text{in}}} \) is thus expressed as:
\begin{equation}
    \mathbf{h} = \mathcal{W}_0 \mathbf{x} + \Delta \mathcal{W} \mathbf{x} = \mathcal{W}_0 \mathbf{x} + \mathbf{B} \mathbf{A} \mathbf{x}.
\end{equation}
The above process can be interpreted as a two-layer linear mapping, where $\mathbf{A}$ serves as a downward projector $\mathcal{P}_\text{down}$ that maps the input $\mathbf{x}$ from a high-dimensional space \( \mathbb{R}^{d_{\text{in}}} \) to a lower-dimensional hidden space \( \mathbb{R}^r \), and $\mathbf{B}$ serves as an upward projector $\mathcal{P}_\text{up}$ that maps back to \(\mathbb{R}^{d_{\text{out}}} \). However, we note that LoRA is constrained by its sequential linear mapping structure, leading to two significant shortcomings: \ding{182} \textbf{insufficient expressiveness}: being a purely linear structure, it requires increasing the hidden dimension to handle more complex incremental weights and improve performance; \ding{183} \textbf{limited training flexibility}: the direct low-rank decomposition induces strong interdependencies between the linear layers, imposing rigid structural constraints that reduce training flexibility \cite{linear_lora_survey}.

\begin{figure*}[!t]
  \centering
  \includegraphics[width=\linewidth]{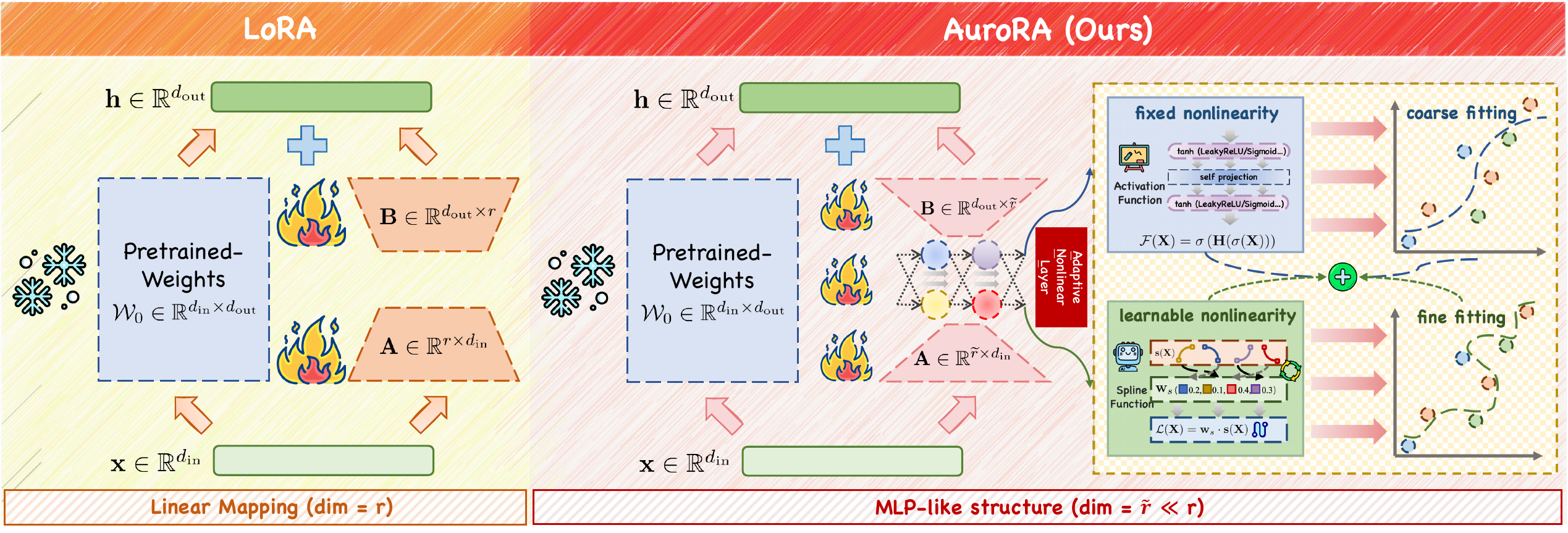}
  \vspace{-1em}
  \caption{A general comparison of LoRA and our \ourmethod. (\textit{\textbf{Left}}) In LoRA, matrices $\mathbf{A}$ and $\mathbf{B}$ act as two linear projectors, forming a two-layer linear mapping with hidden dimension $r$.  
(\textit{\textbf{Right}}) Our \ourmethod extends LoRA by incorporating an adaptive nonlinear layer comprising fixed and learnable nonlinearities, forming an MLP-like structure with significantly reduced hidden dimension $\widetilde{r}$ ($\widetilde{r} \ll r$).}
   \label{fig:overview}
   \vspace{-1em}
\end{figure*}

\vspace{-0.5em}
\subsection{\ourmethod: An MLP-like Structure}
\vspace{-0.5em}
To address these limitations, \ourmethod introduces an Adaptive Nonlinear Layer (ANL) between \( \mathbf{A} \) and \( \mathbf{B} \), modifying the weight update as follows:
\begin{equation}\label{eq:delta}
    \Delta \mathcal{W} = \mathbf{B} \cdot \sigma(\mathbf{A}),
\end{equation}
where $\sigma$ is the element-wise ANL that maps from \(\mathbb{R}^{\widetilde{r}}\) to \(\mathbb{R}^{\widetilde{r}}\). Here, \(\widetilde{r}\) denotes the compressed hidden dimension (\(\widetilde{r} \ll r\)), i.e., the low dimension to which the input is projected by $\mathcal{P}_\text{down}$. Formally, the forward propagation equation in the training phase is given by:
\begin{equation} \label{eq:train-forward}
    \mathbf{h} = \mathcal{W}_0 \mathbf{x} + \mathbf{B} \cdot \sigma(\mathbf{A} \mathbf{x}).
\end{equation}
The introduction of ANL enables \ourmethod to form an MLP (Multilayer Perceptron)-like structure.

\vspace{-0.5em}
\subsection{Adaptive Nonlinear Layer}
\vspace{-0.5em}

Consider an arbitrary input vector \( \mathbf{z} \). After projecting \( \mathbf{z} \) into an \( \widetilde{r} \)-dimensional hidden layer, our objective is to introduce sufficient nonlinearity to capture as many complex relationships as possible within this limited hidden space. To achieve this, we propose the following components: \ding{182} \textbf{\textit{fixed nonlinearity (\( \mathcal{F} \))}}, which utilizes parameter-free nonlinear activation functions to activate neurons in the hidden space, thereby achieving \textit{coarse fitting}; and \ding{183} \textbf{\textit{learnable nonlinearity (\( \mathcal{L} \))}}, which employs parameterized nonlinear functions during the training process of the weight update increments, facilitating \textit{fine fitting}. By combining \ding{182} and \ding{183}, the Adaptive Nonlinear Layer (ANL) can be formally expressed as:
\begin{equation}
    \sigma(\mathbf{Z}) = \mathcal{F}(\mathbf{Z}) + \mathcal{L}(\mathbf{Z}),
\end{equation}
where \( \mathcal{F} \) represents the fixed nonlinear activation, \( \mathcal{L} \) denotes the learnable nonlinear function, and \( \mathbf{Z} \) denotes the input to ANL. We provide a detailed comparison of fixed and learnable nonlinearity in Section \ref{sec:study-ablation}.

For \ding{182}, we adopt widely used activation functions in deep learning, such as ReLU \cite{relu}, sigmoid, and tanh. A detailed comparison of different activation functions and their impact on \ourmethod's performance is provided in Section \ref{sec:study-activation function}. Through our comparative evaluations, \(\tanh\) emerges as the top-performing activation function, and theoretical analysis concurrently ensures its training stability. This preference is consistent with empirical findings in prior studies \cite{adapter,dyt} that demonstrate the robust performance of the \(\tanh\) activation function for large-scale models, leading us to employ \(\tanh\) in our implementation. The depth of the network influences the number of activation functions that can be introduced. Specifically, we introduce a \textit{self-projection} $\mathcal{P}_\text{self} \in \mathbb{R}^{\widetilde{r} \times \widetilde{r}}$ between $\mathcal{P}_\text{down}$ and $\mathcal{P}_\text{up}$, which extends the depth of the standard LoRA structure. Subsequently, we introduce $\tanh$ activation functions between $\mathcal{P}_\text{down}$ and $\mathcal{P}_\text{self}$, and between $\mathcal{P}_\text{self}$ and $\mathcal{P}_\text{up}$. Formally, the fixed nonlinear component is defined as:
\begin{equation}
    \mathcal{F}(\mathbf{Z}) = \tanh\left(\mathbf{H}\left( \tanh(\mathbf{Z})\right)\right),
\end{equation}
where \( \mathbf{H} \in \mathbb{R}^{\widetilde{r} \times \widetilde{r}} \) denotes $\mathcal{P}_\text{self}$.

To achieve \ding{183}, we propose using spline functions to model complex relationships \cite{spline1}. 
Numerous prior studies \cite{kan,spline2,spline3,spline4} have demonstrated that splines are flexible, piecewise polynomial functions capable of approximating a wide range of nonlinear behaviors. 
Specifically, we employ B-spline basis functions to construct the learnable component. Formally, the learnable nonlinear component is defined as:
\begin{equation}
    \mathcal{L}(\mathbf{Z}) = \mathbf{w}_s \cdot \mathbf{s}(\mathbf{Z}),
\end{equation}
where \( \mathbf{w}_s \in \mathbb{R}^{\widetilde{r}} \) is the spline weight vector, and \( \mathbf{s}(\mathbf{Z}) = \sum_{i=1}^{\widetilde{r}} B(z_i) \) represents the spline basis functions applied to each dimension \( z_i \) of \( \mathbf{Z} \). The learnable parameters in this component are the spline weights $\mathbf{w}_s$, which determine the contribution of each basis function $B(z_i)$ to the overall output of \( \mathcal{L}(\mathbf{Z}) \). During training, these weights are iteratively updated to minimize the task-specific loss function.

By introducing \ding{182} and \ding{183} in the hidden layer with dimension $\widetilde{r}$, ANL effectively captures complex relationships without significantly increasing the number of additional parameters. The combination of \textit{coarse fitting} and \textit{fine fitting} enhances the standard LoRA structure, improving its expressive capacity and training flexibility, achieving what we refer to as {\fontfamily{lmtt}\selectfont \textbf{Activate Your Low-Rank Adaptation}}. The complete Adaptive Nonlinear Layer (ANL) developed in our work can then be formally represented as:
\begin{equation}
    \sigma(\mathbf{Z}) = \tanh\left(\mathbf{H}\left( \tanh(\mathbf{Z})\right)\right) + \mathbf{w}_s \cdot \mathbf{s}(\mathbf{Z}).
\end{equation}
Further details and the complete algorithmic workflow of \ourmethod are provided in Appendix \ref{app-complete}.

\vspace{-0.5em}
\subsection{Theoretical Analysis}
\vspace{-0.5em}
In this subsection, we propose two theoretical propositions concerning \ourmethod and analyze its parameter and computational cost. Additionally, we present an intuitive case in Appendix \ref{app-case} to help better understand the role of nonlinearities.

\begin{proposition}
[Lower Approximation Error]
\label{prop:LAE}
Let $M \in \mathbb{R}^{d_{\mathrm{out}}\times d_{\mathrm{in}}}$ with $\mathrm{rank}(M) > r$. 
Define
\[
  \varepsilon_r(M) 
  ~=~
  \inf_{
    U \in \mathbb{R}^{d_{\mathrm{out}}\times r},\,
    V \in \mathbb{R}^{r\times d_{\mathrm{in}}}
  }
  \|\,M - U\,V\|.
\]
Then $\varepsilon_r(M) > 0$, and for our proposed update of the form 
\[
  M_{\mathrm{nonlinear}}(\mathbf{x})
  ~=~
  B\,\sigma\bigl(A\,\mathbf{x}\bigr),
  \quad
  A \in \mathbb{R}^{r\times d_{\mathrm{in}}},
  \;
  B \in \mathbb{R}^{d_{\mathrm{out}}\times r},
\]
where $\sigma$ is our adaptive nonlinear layer, there exists a parameter set $(A^*, B^*, \sigma^*)$ such that
\[
  \bigl\|\,M - M_{\mathrm{nonlinear}}\bigr\| 
  ~\le~
  c\,\varepsilon_r(M),
  \quad
  0 < c < 1.
\]
Hence, the approximation error is strictly below the linear rank-$r$ limit $\varepsilon_r(M)$, using the same rank $r$.
\end{proposition}
\vspace{-0.5em}
$\blacktriangleright$ Proposition \ref{prop:LAE} indicates that, thanks to the introduction of nonlinear mappings, \ourmethod achieves a strictly lower approximation error compared to LoRA at the same rank, meaning that the resulting weight updates are closer to the optimal solution. Furthermore, our empirical results demonstrate that this improvement persists even when further compressing the hidden dimensions of \ourmethod. A rigorous proof, along with technical details and error bounds, is provided in Appendix \ref{app-proof-LAE}.
\begin{proposition}[Gradient Boundedness]\label{prop:gradient-boundedness}
In the \ourmethod, the use of the $\tanh$ activation function and B-spline basis functions results in bounded gradients with respect to both the inputs and the model parameters.
\end{proposition}
\vspace{-0.5em}
$\blacktriangleright$ Proposition \ref{prop:gradient-boundedness} posits that, despite the introduction of fixed and learnable nonlinearities, \ourmethod maintains bounded gradients during training, thereby ensuring training stability. The corresponding proof is provided in Appendix \ref{proof: prop-3}.
\vspace{-0.7em}
\paragraph{Parameter Cost} In Section \ref{sec:intro}, we discussed the relationship between trainable parameters and rank in LoRA, where the number of introduced trainable parameters is $O(r(d_\text{in} + d_\text{out}))$. Here, $d_\text{in}$ and $d_\text{out}$ represent the input and output dimensions, respectively, i.e., $\textsc{params} \propto r$. In \ourmethod, we aim to further compress the parameter count by setting the hidden layer dimension to \(\widetilde{r}=r/k\), where \(k\) is a constant and \(r\) is the optimal rank setting of LoRA. This means that the number of trainable parameters in \ourmethod is \(1/k\) of that in LoRA. In this work, we set \(\widetilde{r}\) to 2, corresponding to values of \(k\) such as 4 and 8. The additional parameters introduced in ANL are of the order \(O(2\widetilde{r}^2)\), which, compared to the significant reduction in parameter count, can be considered negligible.
\vspace{-0.7em}
\paragraph{Computational Cost} The computational complexity of \ourmethod's forward pass in the training phase, $\Delta \mathbf{h} = \mathbf{B} \sigma\cdot(\mathbf{A}\mathbf{x})$, is analyzed as follows. Let $b$ denote the batch size, $d_{in}$ and $d_{out}$ the input/output feature dimensions, $r$ the rank, and $G$ the collective B-spline parameters (a small constant, $G=O(r)$). The linear projections by $\mathbf{A} \in \mathbb{R}^{r \times d_{in}}$ and $\mathbf{B} \in \mathbb{R}^{d_{out} \times r}$ incur complexities of $O(b d_{in} r)$ and $O(b r d_{out})$, respectively. The intermediate fixed and learnable non-linearities, $\sigma(\cdot)$, each contribute an additional $O(b r^2)$ term (with the learnable component's $O(b r G)$ complexity simplifying due to $G=O(r)$). Consequently, the total complexity for \ourmethod is $O(b(d_{in}r + 2r^2 + r d_{out}))$. Given the standard low-rank setting where $r \ll \min(d_{in}, d_{out})$, the quadratic overhead $O(b r^2)$ introduced by the non-linearities is negligible compared to the dominant linear terms, thus maintaining a computational footprint comparable to that of LoRA.

\vspace{-0.5em}
\section{Experiments}
\vspace{-0.5em}
In this section, we conduct extensive experiments to answer the following research questions: ($\boldsymbol{\mathcal{RQ}1}$) Can \ourmethod effectively achieve efficiency in NLP tasks? ($\boldsymbol{\mathcal{RQ}2}$) Can \ourmethod effectively achieve efficiency in CV tasks? ($\boldsymbol{\mathcal{RQ}3}$) What are the respective roles of fixed and learnable nonlinearity? ($\boldsymbol{\mathcal{RQ}4}$) How do different activation functions in fixed nonlinearity affect performance? ($\boldsymbol{\mathcal{RQ}5}$) How does \ourmethod's sensitivity to rank compare to that of linear LoRA variants? \footnote{The source code is available at \href{https://github.com/ins1stenc3/AuroRA}{here}.}

\vspace{-0.5em}
\subsection{Experimental Setup}
\vspace{-0.5em}
\subsubsection{Datasets and Pre-Trained Models}
\vspace{-0.5em}
\paragraph{Datasets} For our experiments, we evaluate the ability of \ourmethod to achieve parameter-efficient fine-tuning using four categories of datasets spanning both NLP and CV domains: $\blacksquare$ \textbf{Natural Language Understanding}: We employ GLUE (General Language Understanding Evaluation) \cite{glue}, a widely used multi-task benchmark in NLU, which includes datasets such as SST-2, MRPC, CoLA, QNLI, RTE, and STS-B. The evaluation metrics are as follows: CoLA is assessed using Matthew's correlation coefficient, STS-B with Pearson's correlation coefficient, and accuracy is used for the other tasks. $\blacksquare$ \textbf{Commonsense Reasoning}: We use a collection of commonly used datasets, including BoolQ \cite{boolq}, PIQA \cite{piqa}, SocialIQA \cite{siqa}, HellaSwag \cite{hellas}, WinoGrande \cite{winog}, ARC-e, ARC-c \cite{arc}, and OpenBookQA \cite{obqa}. For fair comparison, we follow the setup proposed by \cite{moslora}, fine-tuning the pretrained models on the Commonsense170K dataset, which serves as a mixture of the aforementioned benchmark datasets. We then evaluate using accuracy as the performance metric. $\blacksquare$ \textbf{Image Classification}: We use five datasets with small label spaces—OxfordPets \cite{pets}, CIFAR-10 \cite{cifar}, DTD \cite{dtd}, EuroSAT \cite{eurosat}, and RESISC45 \cite{resisc}, and three datasets with large label spaces, namely StanfordCars \cite{cars}, FGVC \cite{fgvc}, and CIFAR-100 \cite{cifar}. $\blacksquare$ \textbf{Subject-Driven Generation}: Following \cite{dreambooth}, we use the DreamBooth dataset. More detailed descriptions of the datasets can be found in Appendix \ref{appendix-dataset-glue}, \ref{appendix-dataset-cr}, \ref{appendix-dataset-icl}.
\vspace{-0.5em}
\paragraph{Pre-Trained Models} We focus on a selection of representative pretrained models, including RoBERTa (Base \& Large) \cite{roberta}, LLAMA3-8B \cite{llama3}, ViT (Base \& Large) \cite{vit} and SDXL \cite{sdxl-paper}.

\begin{table*}[!t]
\centering
\caption{We report the performance of different fine-tuning methods on six datasets of the GLUE benchmark, using RoBERTa-Base and RoBERTa-Large models. For CoLA, we report the Matthew's Correlation Coefficient (MCC); for STS-B, we report the Pearson Correlation Coefficient (PCC); and for all other tasks, we report accuracy (Acc.). The reported results are the medians of five runs, each using a different random seed. * indicates numbers published in prior works. The best results are highlighted in \textbf{bold}, and the runners-up are \underline{underlined}. For all six datasets, higher values are considered better for all metrics.}
\vspace{0.5em}
\label{tab:result-nlu}
\renewcommand\tabcolsep{5.3pt}
\renewcommand\arraystretch{1.1}

\resizebox{\linewidth}{!}{
\begin{tabular}{c|l|ccccccc|c}
\Xhline{1.2pt}
\rowcolor{CadetBlue!20} 
\textbf{Model} & \textbf{Method} &  \textbf{SST-2} & \textbf{MRPC} & \textbf{CoLA} & \textbf{QNLI} & \textbf{RTE} & \textbf{STS-B} & \textbf{Avg.} & \textbf{Params.}\\
\Xhline{1.2pt}

\multirow{10}{*}{\rotatebox{90} {RoBERTa-Base}} 
& Full Fine-Tuning* & 94.8 & 90.2 & 63.6 & 92.8 & 78.7 & 91.2 & 85.2 & 125M\\
\cline{2-10} 
\noalign{\vskip 0.1mm}

& \cellcolor{gray!10}BitFit* & \cellcolor{gray!10}93.7\blue{1.1} & \cellcolor{gray!10}\textbf{92.7}\red{2.5} & \cellcolor{gray!10}62.0\blue{1.6} & \cellcolor{gray!10}91.8\blue{1.0} & \cellcolor{gray!10}\underline{81.5}\red{2.8} & \cellcolor{gray!10}90.8\blue{0.4} & \cellcolor{gray!10}85.4\red{0.2} & \cellcolor{gray!10}0.1M\\

& $\text{Adapter}^{\text{D}}$* & 94.7\blue{0.1} & 88.4\blue{1.8} & 62.6\blue{1.0} & 93.0\red{0.2} & 75.9\blue{2.8} & 90.3\blue{0.9} & 84.2\blue{1.0} & 0.9M \\

& \cellcolor{gray!10}LoRA* & \cellcolor{gray!10}\underline{95.1}\red{0.3} & \cellcolor{gray!10}89.7\blue{0.5} & \cellcolor{gray!10}63.4\blue{0.2} & \cellcolor{gray!10}\underline{93.3}\red{0.5} & \cellcolor{gray!10}78.4\blue{0.3} & \cellcolor{gray!10}\textbf{91.5}\red{0.3} & \cellcolor{gray!10}85.2\red{0.0} & \cellcolor{gray!10}0.3M\\

& AdaLoRA* & 94.5\blue{0.3} & 88.7\blue{1.5} & 62.0\blue{1.6} & 93.1\red{0.3} & 81.0\red{2.3} & 90.5\blue{0.7} & 85.0\blue{0.2} & 0.3M\\

& \cellcolor{gray!10}DyLoRA* & \cellcolor{gray!10}94.3\blue{0.5} & \cellcolor{gray!10}89.5\blue{0.7} & \cellcolor{gray!10}61.1\blue{2.5} & \cellcolor{gray!10}92.2\blue{0.6} & \cellcolor{gray!10}78.7\red{0.0} & \cellcolor{gray!10}91.1\blue{0.1} & \cellcolor{gray!10}84.5\blue{1.3} & \cellcolor{gray!10}0.3M\\

& FourierFT* & 94.2\blue{0.6} & 90.0\blue{0.2} & 63.8\red{0.2} & 92.2\blue{0.6} & 79.1\red{0.4} & 90.8\blue{0.4} & 85.0\blue{0.2} & 0.024M\\

& \cellcolor{gray!10}LoRA-drop* & \cellcolor{gray!10}94.5\blue{0.3} & \cellcolor{gray!10}89.5\blue{0.7} & \cellcolor{gray!10}62.9\blue{0.7} & \cellcolor{gray!10}93.1\red{0.3} & \cellcolor{gray!10}81.4\red{2.7} & \cellcolor{gray!10}91.0\blue{0.2} & \cellcolor{gray!10}85.4\red{0.2} & \cellcolor{gray!10}0.15M \\

& DoRA* & 95.0\red{0.2} & 89.7\red{0.5} & \underline{64.9}\red{1.3} & 92.9\red{0.1} & 79.2\red{0.5} & \underline{91.3}\red{0.1} & \underline{85.5}\red{0.3} & 0.3M\\

& \cellcolor{gray!10}\ourmethod  & \cellcolor{gray!10}\textbf{95.2}\darkred{\textbf{0.4}} & \cellcolor{gray!10}\underline{91.9}\darkred{\textbf{1.7}} & \cellcolor{gray!10}\textbf{65.1}\darkred{\textbf{1.5}} & \cellcolor{gray!10}\textbf{93.4}\darkred{\textbf{0.6}} & \cellcolor{gray!10}\textbf{85.2}\darkred{\textbf{6.5}} & \cellcolor{gray!10}\textbf{91.5}\darkred{\textbf{0.3}} & \cellcolor{gray!10}\textbf{87.1}\darkred{\textbf{1.7}} & \cellcolor{gray!10}0.075M\\

\hline

\multirow{6}{*}{\rotatebox{90} {RoBERTa-Large}} 
& Full Fine-Tuning*  & \underline{96.4} & \underline{90.9} & 68.0 & 94.7 & 86.6 & \underline{92.4} & \underline{88.2} & 356M\\
\cline{2-10} 
\noalign{\vskip 0.1mm}

& \cellcolor{gray!10}$\text{Adapter}^{\text{P}}$*  & \cellcolor{gray!10}96.1\blue{0.3} & \cellcolor{gray!10}90.2\blue{0.7} & \cellcolor{gray!10}\underline{68.3}\red{0.3} & \cellcolor{gray!10}\underline{94.8}\red{0.1} & \cellcolor{gray!10}83.8\blue{2.8} & \cellcolor{gray!10}92.1\blue{0.3} & \cellcolor{gray!10}87.6\blue{0.6} & \cellcolor{gray!10}3M\\

& $\text{Adapter}^{\text{H}}$*  & 96.2\blue{0.2} & 88.7\blue{2.2} & 66.5\blue{1.5} & 94.7\red{0.0} & 83.4\blue{3.2} & 91.0\blue{1.2} & 86.8\blue{1.4} & 6M\\

& \cellcolor{gray!10}LoRA*  & \cellcolor{gray!10}96.2\blue{0.2} & \cellcolor{gray!10}90.2\blue{0.7} & \cellcolor{gray!10}68.2\red{0.2} & \cellcolor{gray!10}\underline{94.8}\red{0.1} & \cellcolor{gray!10}85.2\blue{1.4} & \cellcolor{gray!10}92.3\blue{0.1} & \cellcolor{gray!10}87.8\blue{0.4} & \cellcolor{gray!10}0.8M\\

& FourierFT*  & 96.0\blue{0.4} & \underline{90.9}\red{0.0} & 67.1\blue{0.9} & 94.4\blue{0.3} & \underline{87.4}\red{0.8} & 91.9\blue{0.5} & 88.0\blue{0.2} & 0.048M\\

& \cellcolor{gray!10}\ourmethod  & \cellcolor{gray!10}\textbf{96.6}\darkred{\textbf{0.2}} & \cellcolor{gray!10}\textbf{91.2}\darkred{\textbf{0.3}} & \cellcolor{gray!10}\textbf{69.2}\darkred{\textbf{1.2}} & \cellcolor{gray!10}\textbf{95.0}\darkred{\textbf{0.3}} & \cellcolor{gray!10}\textbf{89.9}\darkred{\textbf{3.3}} & \cellcolor{gray!10}\textbf{92.5}\darkred{\textbf{0.1}} & \cellcolor{gray!10}\textbf{89.1}\darkred{\textbf{0.9}} & \cellcolor{gray!10}0.2M\\

\Xhline{1.2pt}
\end{tabular}
}
\vspace{-1.0em}
\end{table*}

\begin{table*}[!t]
\centering
\caption{Commonsense reasoning evaluation results for LLaMA3-8B on eight tasks. * indicates numbers taken from \cite{milora}. The best results are highlighted in \textbf{bold}, and the runners-up are \underline{underlined}. For all eight tasks, higher values are considered better.}
\vspace{0.5em}
\label{tab:result-cr}
\renewcommand\tabcolsep{5.3pt}
\renewcommand\arraystretch{1.1}

\resizebox{\linewidth}{!}{
\begin{tabular}{l|r|ccccccccc}
\Xhline{1.2pt}
\rowcolor{CadetBlue!20} 
\textbf{Method} & \textbf{Params.} &  \textbf{BoolQ} & \textbf{PIQA} & \textbf{SIQA} & \textbf{HellaSwag} & \textbf{WinoGrande} & \textbf{ARC-e} & \textbf{ARC-c} & \textbf{OBQA} & \textbf{Avg.} \\
\Xhline{1.2pt}

LoRA* & 56.6M & \underline{70.8} & 85.2 & \textbf{79.9} & 91.7 & \underline{84.3} & 84.2 & 71.2 & 79.0 & 80.8 \\

\cellcolor{gray!10}PiSSA* & \cellcolor{gray!10}83.8M & \cellcolor{gray!10}67.1\blue{3.7} & \cellcolor{gray!10}81.1\blue{4.1} & \cellcolor{gray!10}77.2\blue{2.7} & \cellcolor{gray!10}83.6\blue{8.1} & \cellcolor{gray!10}78.9\blue{5.4} & \cellcolor{gray!10}77.7\blue{6.5} & \cellcolor{gray!10}63.2\blue{8.0} & \cellcolor{gray!10}74.6\blue{5.4} & \cellcolor{gray!10}75.4\blue{5.4} \\

MiLoRA* & 56.6M & 68.8\blue{2.0} & \underline{86.7}\red{1.5} & 77.2\blue{2.7} & \underline{92.9}\red{1.2} & \textbf{85.6}\red{1.3} & \underline{86.8}\red{2.6} & \underline{75.5}\red{4.3} & \underline{81.8}\red{2.8} & \underline{81.9}\red{1.1} \\

\cellcolor{gray!10}\ourmethod & \cellcolor{gray!10}3.5M & \cellcolor{gray!10}\textbf{72.5}\darkred{\textbf{1.7}} & \cellcolor{gray!10}\textbf{87.4}\darkred{\textbf{2.2}} & \cellcolor{gray!10}\underline{79.0}\darkblue{\textbf{0.9}} & \cellcolor{gray!10}\textbf{94.2}\darkred{\textbf{2.5}} & \cellcolor{gray!10}83.0\darkblue{\textbf{1.3}} & \cellcolor{gray!10}\textbf{89.3}\darkred{\textbf{5.1}} & \cellcolor{gray!10}\textbf{78.8}\darkred{\textbf{7.6}} & \cellcolor{gray!10}\textbf{84.8}\darkred{\textbf{5.8}} & \cellcolor{gray!10}\textbf{83.6}\darkred{\textbf{2.8}} \\
\Xhline{1.2pt}
\end{tabular}
}
\vspace{-1.5em}
\end{table*}

\vspace{-0.5em}
\subsubsection{Baselines}
\vspace{-0.5em}
In the baseline evaluation, we adopt a range of representative and competitive fine-tuning methods, categorized into three groups: Full Fine-Tuning, PEFT methods, and LoRA variants. The PEFT methods we use include BitFit \cite{bitfit}, Adapter$^{\text{H}}$ \cite{adapter}, Adapter$^{\text{D}}$ \cite{adapterdrop}, Adapter$^{\text{P}}$ \cite{adapaterfusion}, and LoRA \cite{lora}. For the LoRA variants, we consider AdaLoRA \cite{adalora}, DyLoRA \cite{dylora}, FourierFT \cite{fourierft}, LoRA-drop \cite{loradrop}, DoRA \cite{dora_baseline}, MoSLoRA \cite{moslora}, PiSSA \cite{pissa} and MiLoRA \cite{milora}.

\vspace{-0.5em}
\subsection{\ourmethod Achieves Efficiency in NLP Tasks (\texorpdfstring{$\boldsymbol{\mathcal{RQ}1}$}{})}
\vspace{-0.5em}
To answer $\mathcal{RQ}$1, we design two tasks: Natural Language Understanding (NLU) and Commonsense Reasoning. In the NLU task, we select RoBERTa-Base and RoBERTa-Large \cite{roberta} as pretrained models and compare \ourmethod with \textbf{ten} other widely-used fine-tuning methods across all six datasets of the GLUE benchmark \cite{glue}. The results of this extensive comparison are shown in Table \ref{tab:result-nlu}, with additional hyperparameter configuration details provided in Appendix \ref{appendix-hyper-nlu}. Following \cite{lora}, we fine-tune only the query and value weights of each transformer block, while fully fine-tuning the classification head. For the commonsense reasoning task, we select LLaMA3-8B \cite{llama3} as the base model and compare \ourmethod with LoRA and two other LoRA variants (PiSSA \cite{pissa} and MiLoRA \cite{milora}). The results are shown in Table \ref{tab:result-cr}. The relevant hyperparameters are listed in Appendix \ref{appendix-hyper-cr}. Our observations can be summarized as follows:

\begin{table*}[!t]
\centering
\caption{Fine-tuning results with ViT Base and Large models on different image classification datasets. We report the accuracy (\%) after 10 epochs. Avg. represents the average accuracy across all datasets for each method. * indicates numbers taken from \cite{fourierft}. The best results are highlighted in \textbf{bold}, and the runners-up are \underline{underlined} (excluding full fine-tuning).}
\vspace{0.5em}
\label{tab:result-icl}
\renewcommand\tabcolsep{5.3pt}
\renewcommand\arraystretch{1.1}

\resizebox{\linewidth}{!}{
\begin{tabular}{c|l|c|ccccccccc}
\Xhline{1.2pt}
\rowcolor{CadetBlue!20} 
\textbf{Model} & \textbf{Method} & \textbf{Params.} &  \textbf{OxfordPets} & \textbf{StanfordCars} & \textbf{CIFAR10} & \textbf{DTD} & \textbf{EuroSAT} & \textbf{FGVC} & \textbf{RESISC45} & \textbf{CIFAR100} & \textbf{Avg.} \\
\Xhline{1.2pt}

\multirow{5}{*}{\rotatebox{90} {ViT-Base}} 

& Full Fine-Tuning* & 85.8M & 93.1 & 79.8 & 98.9 & 77.7 & 99.1 & 54.8 & 96.1 & 92.4 & 86.5 \\
\cline{2-12} 
\noalign{\vskip 0.1mm}

& \cellcolor{gray!10}Linear Probing* & \cellcolor{gray!10}- & \cellcolor{gray!10}90.3\blue{2.8} & \cellcolor{gray!10}25.8\blue{54.0} & \cellcolor{gray!10}96.4\blue{2.5} & \cellcolor{gray!10}69.8\blue{7.9} & \cellcolor{gray!10}88.7\blue{10.4} & \cellcolor{gray!10}17.4\blue{37.4} & \cellcolor{gray!10}74.2\blue{21.9} & \cellcolor{gray!10}84.3\blue{8.1} & \cellcolor{gray!10}68.4\blue{18.1} \\

& LoRA* & 581K & \underline{93.2}\red{0.1} & 45.4\blue{34.4} & \textbf{98.8}\blue{0.1} & 75.0\blue{2.7} & \underline{98.4}\blue{0.7} & 25.2\blue{29.6} & 92.7\blue{3.4} & \textbf{92.0}\blue{0.4} & 77.6\blue{8.9} \\

& \cellcolor{gray!10}FourierFT* & \cellcolor{gray!10}239K & \cellcolor{gray!10}93.1\red{0.0} & \cellcolor{gray!10}\underline{56.4}\blue{23.4} & \cellcolor{gray!10}\underline{98.7}\blue{0.2} & \cellcolor{gray!10}\underline{77.3}\blue{0.4} & \cellcolor{gray!10}\textbf{98.8}\blue{0.3} & \cellcolor{gray!10}\underline{32.4}\blue{22.4} & \cellcolor{gray!10}\textbf{94.3}\blue{1.8} & \cellcolor{gray!10}\underline{91.5}\blue{0.9} & \cellcolor{gray!10}\underline{80.3}\blue{6.2} \\

& \ourmethod & 74K & \textbf{93.9}\darkred{\textbf{0.8}} & \textbf{75.7}\darkblue{\textbf{4.1}} & \textbf{98.8}\darkblue{\textbf{0.1}} & \textbf{79.6}\darkred{\textbf{1.9}} & \textbf{98.8}\darkblue{\textbf{0.3}} & \textbf{48.2}\darkblue{\textbf{6.6}} & \underline{93.6}\darkblue{\textbf{2.5}} & \textbf{92.0}\darkblue{\textbf{0.4}} & \textbf{85.1}\darkblue{\textbf{1.4}} \\

\hline

\multirow{5}{*}{\rotatebox{90} {ViT-Large}} 
& \cellcolor{gray!10}Full Fine-Tuning* & \cellcolor{gray!10}303.3M & \cellcolor{gray!10}94.4 & \cellcolor{gray!10}88.9 & \cellcolor{gray!10}99.2 & \cellcolor{gray!10}81.8 & \cellcolor{gray!10}99.0 & \cellcolor{gray!10}68.3 & \cellcolor{gray!10}96.4 & \cellcolor{gray!10}93.6 & \cellcolor{gray!10}90.2 \\
\cline{2-12} 
\noalign{\vskip 0.1mm}

& Linear Probing* & - & 91.1\blue{3.3} & 37.9\blue{51.0} & \underline{97.8}\blue{1.4} & 73.3\blue{8.5} & 92.6\blue{6.4} & 24.6\blue{43.7} & 82.0\blue{14.4} & 84.3\blue{9.3} & 73.0\blue{17.2} \\

& \cellcolor{gray!10}LoRA* & \cellcolor{gray!10}1.57M & \cellcolor{gray!10}\underline{94.8}\red{0.4} & \cellcolor{gray!10}73.3\blue{15.6} & \cellcolor{gray!10}\textbf{99.1}\blue{0.1} & \cellcolor{gray!10}81.8\red{0.0} & \cellcolor{gray!10}98.6\blue{0.4} & \cellcolor{gray!10}42.3\blue{26.0} & \cellcolor{gray!10}94.7\blue{1.7} & \cellcolor{gray!10}\textbf{94.9}\red{1.3} & \cellcolor{gray!10}84.9\blue{5.3} \\

& FourierFT* & 480K & \underline{94.8}\red{0.4} & \underline{79.1}\blue{9.8} & \textbf{99.1}\blue{0.1} & \underline{81.9}\red{0.1} & \underline{98.7}\blue{0.3} & \underline{51.3}\blue{17.0} & \textbf{95.2}\blue{1.2} & \underline{93.4}\blue{0.2} & \underline{86.7}\blue{3.5} \\

& \cellcolor{gray!10}\ourmethod & \cellcolor{gray!10}197K & \cellcolor{gray!10}\textbf{94.9}\darkred{\textbf{0.5}} & \cellcolor{gray!10}\textbf{82.5}\darkblue{\textbf{6.4}} & \cellcolor{gray!10}\textbf{99.1}\darkblue{\textbf{0.1}} & \cellcolor{gray!10}\textbf{82.1}\darkred{\textbf{0.3}} & \cellcolor{gray!10}\textbf{98.9}\darkblue{\textbf{0.1}} & \cellcolor{gray!10}\textbf{59.8}\darkblue{\textbf{8.5}} & \cellcolor{gray!10}\underline{94.9}\darkblue{\textbf{1.5}} & \cellcolor{gray!10}93.3\darkblue{\textbf{0.3}} & \cellcolor{gray!10}\textbf{88.2}\darkblue{\textbf{2.0}} \\

\Xhline{1.2pt}
\end{tabular}
}
\vspace{-1.3em}
\end{table*}

\vspace{-1em}
\paragraph{Obs. \ding{182} \ourmethod demonstrates strong efficiency in NLP tasks.} It is evident that \ourmethod outperforms the baseline across all datasets and pretrained models in both tasks. Compared to Full Fine-Tuning, \ourmethod achieves a performance improvement ranging from $0.1\% \sim 8.3\%$ while using only $0.04\% \sim 0.06\%$ of the total parameters. Compared to PEFT baselines, including LoRA, \ourmethod achieves a performance improvement of up to $24.7\%$ and an average improvement of $1.25\% \sim 10.88\%$, using only $6.25\% \sim 25\%$ of the parameters. Specifically, in the commonsense reasoning task using LLaMA3-8B as the pretrained model, \ourmethod achieves a significant $10.7\%$ performance boost on \textsc{arc-c} with just $6.25\%$ of LoRA's parameter budget. In the NLU task, although \ourmethod uses more parameters than FourierFT, it demonstrates significant performance gains across all pretrained models and datasets. For instance, using RoBERTa-Base, \ourmethod improves performance by $7.7\%$ on \textsc{rte}.

\vspace{-1em}
\paragraph{Obs. \ding{183} \ourmethod can be scaled up to fine-tune large pretrained models.} \ourmethod scales effectively to fine-tuning larger pretrained models. In the NLU task, when the pretrained model changes to RoBERTa-Large from Base, nearly all PEFT methods show a performance drop compared to Full Fine-Tuning, with the largest decrease reaching $3.7\%$. In contrast, \ourmethod still achieves performance improvements of $0.1\% \sim 3.8\%$ across all datasets. In the commonsense reasoning task, when the model size increases to 8B, \ourmethod continues to outperform LoRA by $2.4\% \sim 10.7\%$.

\vspace{-0.5em}
\subsection{\ourmethod Achieves Efficiency in CV Tasks (\texorpdfstring{$\boldsymbol{\mathcal{RQ}2}$}{})}
\vspace{-0.5em}
To answer $\mathcal{RQ}$2, we design two tasks: Image Classification and Subject-Driven Image Generation. In the image classification task, following \cite{fourierft}, we select ViT-Base and ViT-Large \cite{vit}, two popular CV foundation models, which are pretrained on the ImageNet-21K \cite{imagenet21k} dataset. We then compare \ourmethod with Full Fine-Tuning, Linear Probing (fine-tuning only the classification head), LoRA, and FourierFT. The results are presented in Table \ref{tab:result-icl}, with more implementation details available in Appendix \ref{appendix-hyper-icl}. In the subject-driven image generation task \cite{dreambooth}, following \cite{fourierft} and \cite{moslora}, we use the \href{https://huggingface.co/stabilityai/stable-diffusion-xl-base-1.0}{SDXL} model \cite{sdxl-paper} as our backbone, and then fine-tune it using both LoRA and \ourmethod. The objective is to generate images based on specified prompts for a particular subject, which is defined using a set of reference images. Initially, we fine-tune a text-to-image model by pairing the input images with text prompts that include a unique identifier (e.g., ``A photo of a [V] dog''). Subsequently, the model can generate images corresponding to other prompts that incorporate the same unique identifier, thereby producing images of the defined subject. The results are presented in Figure \ref{fig:sd}, and more generated cases are in Appendix \ref{sec:app-sd}. Our observations can be summarized as follows:
\vspace{-1em}
\paragraph{Obs. \ding{184} \ourmethod achieves the best performance, excluding Full Fine-Tuning, with the least number of parameters.} It is evident that \ourmethod outperforms all other PEFT baseline methods across all eight datasets with the lowest parameter count ($12.7\%$ of LoRA and $31.0\%$ of FourierFT) when using both the Base and Large models. Compared to Full Fine-Tuning, \ourmethod uses only $0.086\%$ of the parameters and achieves a performance improvement of $0.4\% \sim 2.4\%$ on some datasets, with only a $1.6\% \sim 2.2\%$ gap in average performance. When using ViT-Base on \textsc{stanfordcars}, other baselines show a significant performance drop of $29.3\% \sim 67.7\%$ compared to Full Fine-Tuning. In contrast, \ourmethod only experiences a moderate drop of $5.1\%$. Compared to PEFT methods, \ourmethod achieves an average performance improvement of $1.73\% \sim 9.66\%$.

\begin{figure*}[!t]
  \centering
  \includegraphics[width=1\linewidth]{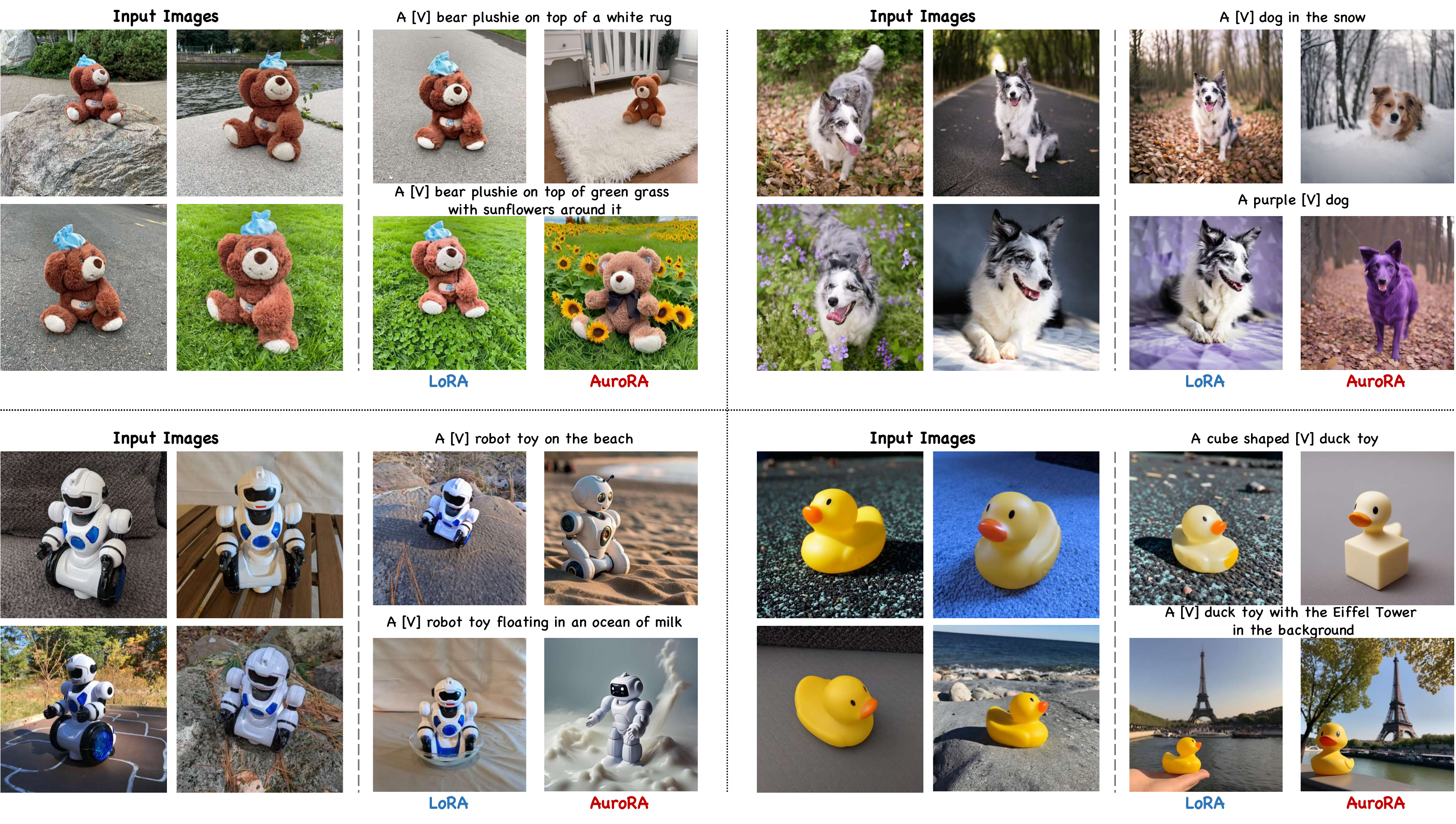}
  \vspace{-2em}
  \caption{Results of LoRA and \ourmethod in the subject-driven image generation task. \ourmethod aligns better with the prompt.}
   \label{fig:sd}
   \vspace{-1.5em}
\end{figure*}
\vspace{-1em}
\paragraph{Obs. \ding{185} \ourmethod demonstrates stronger adaptability in the text-to-image domain.} We observe that in the subject-driven image generation task, \ourmethod aligns better with the environment specified in the prompt. Specifically, when given the prompt ``A [V] bear plushie on top of green grass with sunflowers around it'', LoRA generates an environment with only green grass but no sunflowers. In contrast, \ourmethod successfully generates green grass with sunflowers.

\vspace{-0.5em}
\subsection{Study}
\vspace{-0.5em}
\paragraph{Ablation Study ($\boldsymbol{\mathcal{RQ}3}$)}\label{sec:study-ablation} 
To evaluate the contribution of different modules in \ourmethod, we introduce two variants: (1) \ourmethod w/o $\mathcal{F}$, and (2) \ourmethod w/o $\mathcal{L}$, which correspond to the removal of the fixed and learnable nonlinearity in \ourmethod, respectively. We compare these two variants with \ourmethod by fine-tuning ViT-Base on \textsc{oxfordpets}, \textsc{cifar10}, \textsc{dtd}, and \textsc{eurosat} in the image classification task. From Table \ref{tab:study-ablation}, we observe that: \ding{182} removing any component results in a performance drop for \ourmethod; \ding{183} \ourmethod w/o $\mathcal{L}$ consistently underperforms across all datasets, indicating that the learnable nonlinearity plays a more crucial role in the success of our method. Specifically, the learnable nonlinearity enables fine fitting, while the fixed nonlinearity contributes to coarse fitting.
\vspace{-0.5em} 
\begin{table}[!h]
\centering
\begin{minipage}[t]{0.48\linewidth}
  \centering
  \vspace{-0.5em}
  \caption{Comparison of different settings.}\vspace{-0.5em}
  \label{tab:study-ablation}
  \resizebox{\linewidth}{!}{
  \begin{tabular}{l|cccccccc}
\toprule
\textbf{Setting} & \textbf{OxfordPets} & \textbf{CIFAR10} & \textbf{DTD} & \textbf{EuroSAT} \\
\midrule
\ourmethod & 93.9 & 98.8 & 79.6 & 98.8 \\
\ourmethod w/o $\mathcal{F}$ & 93.3 $_{\downarrow 0.6}$ & 98.4 $_{\downarrow 0.4}$ & 78.9 $_{\downarrow 0.7}$ & 98.3 $_{\downarrow 0.5}$ \\
\ourmethod w/o $\mathcal{L}$ & 93.1 $_{\downarrow 0.8}$ & 98.2 $_{\downarrow 0.6}$ & 77.8 $_{\downarrow 1.8}$ & 98.0 $_{\downarrow 0.8}$ \\
\bottomrule
\end{tabular}
}
\end{minipage}%
\hfill
\begin{minipage}[t]{0.48\linewidth}
  \centering
  \vspace{-0.5em}
  \caption{Comparison of different activation functions.}\vspace{-0.5em}
  \label{tab:study-activation}
  \resizebox{\linewidth}{!}{
\begin{tabular}{l|cccc}
\toprule
\textbf{Setting} & \textbf{StanfordCars} & \textbf{FGVC} & \textbf{RESISC45} & \textbf{CIFAR100} \\
\midrule
\ourmethod & 75.7 & 48.2 & 93.6 & 92.0 \\
\ourmethod -lr & 75.6 $_{\downarrow 0.1}$ & 47.8 $_{\downarrow 0.4}$ & 93.4 $_{\downarrow 0.2}$ & 91.9 $_{\downarrow 0.1}$ \\
\ourmethod -sm & 75.2 $_{\downarrow 0.5}$ & 47.7 $_{\downarrow 0.5}$ & 92.9 $_{\downarrow 0.7}$ & 91.7 $_{\downarrow 0.3}$ \\
\bottomrule
\end{tabular}
}
\end{minipage}
\end{table}

\vspace{-2.5em}
\paragraph{Effect of Activation Function ($\boldsymbol{\mathcal{RQ}4}$)} \label{sec:study-activation function} 
We investigate the impact of the choice of activation function in the fixed nonlinearity on \ourmethod's performance. Specifically, we introduce two variants: (1) \ourmethod-lr, and (2) \ourmethod-sm, which correspond to replacing the activation function in the fixed nonlinearity ($\tanh$) with LeakyReLU and Sigmoid, respectively. We compare these variants with \ourmethod by fine-tuning the ViT-Base model on \textsc{stanfordcars}, \textsc{fgvc}, \textsc{resisc45}, and \textsc{cifar100} in the image classification task. From Table \ref{tab:study-activation}, we observe that Sigmoid results in the lowest performance, while $\tanh$ achieves the highest performance. Therefore, we choose $\tanh$ as the activation function for fixed nonlinearity in all our experiments.
\vspace{-1.0em}
\paragraph{Sensitivity to Rank \& Comparison with Linear LoRA Variants ($\boldsymbol{\mathcal{RQ}5}$)} 
To further investigate the impact of introducing nonlinearity, we examine its sensitivity to rank and compare it with the linear LoRA variant under identical experimental settings. Specifically, we select LLaMA3-8B as the pretrained model and fine-tune it using \ourmethod, MoSLoRA, and LoRA, varying the rank among $\{2, 4, 8, 16\}$. We evaluate their performance across four datasets. From Figure \ref{fig:vslinear}, we observe the following: \ding{182} the introduction of nonlinearity results in smaller performance fluctuations as the rank varies, i.e., more robustness to rank; \ding{183} \ourmethod consistently outperforms across almost all rank settings and datasets, indicating that incorporating nonlinearity further enhances the model's expressiveness compared to linear approaches.

\begin{figure*}[!t]
  \centering
  \includegraphics[width=1\linewidth]{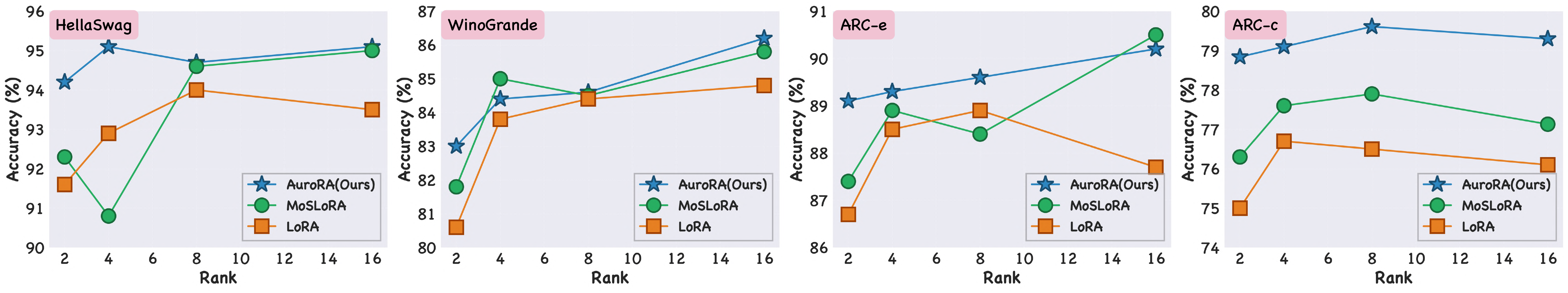}
  \vspace{-2em}
  \caption{Performance comparison of different methods with varying ranks. We use LLaMA 3-8B as the pretrained model and fine-tune it using \ourmethod, MoSLoRA, and LoRA methods on the \textsc{hellaswag}, \textsc{winogrande}, \textsc{arc-e} and \textsc{arc-c} datasets, with ranks $\{2,4,8,16\}$.}
   \label{fig:vslinear}
   \vspace{-1.2em}
\end{figure*}

\vspace{-0.5em}
\section{Related Work}
\vspace{-0.5em}
\subsection{Parameter-Efficient Fine-Tuning}
\vspace{-0.5em}

Parameter-Efficient Fine-Tuning (PEFT) has emerged as a pivotal strategy for addressing the computational challenges associated with fine-tuning large-scale pretrained models. PEFT methodologies can be broadly categorized into:
\ding{182} \textbf{Additive PEFT} approaches introduce new, trainable modules to a frozen base model \cite{peft1, ssf, adapter,adapter1,additive-3,additive-4,gpt}. Common strategies include adapter-based techniques, such as AdapterFusion \cite{adapaterfusion} and Hyperformer \cite{hyperformer}; prompt-based methods, like Prefix-tuning \cite{prefix} and p-tuning v2 \cite{ptuningv2}.
\ding{183} \textbf{Selective PEFT} methods optimize a chosen subset of a pretrained model's parameters while keeping the majority frozen \cite{select-1,select-2,select-3,select-5,select-6,hft}. This selection is often achieved through unstructured masking based on criteria like parameter significance, as seen in FishMask \cite{select-4} and Child-tuning \cite{child}, or via structured techniques that group parameters, such as Bitfit \cite{bitfit} and SPT \cite{spt}.
\ding{184} \textbf{Reparameterized PEFT} techniques transform model weights into more efficient, often low-rank, representations during fine-tuning, without altering the core architecture for inference \cite{vera,dylora,moelora-1,moelora-2,moelora-3}. A prominent example is LoRA \cite{lora}, which introduces low-rank matrices for updates. \ding{185} \textbf{Memory-Efficient PEFT} methods focus on reducing the memory footprint of fine-tuning by optimizing the training dynamics rather than the model architecture \cite{badam,lisa,ows,galore}. A representative example is GaLore \cite{galore}, which projects gradients into low-rank subspaces to lower optimizer-state memory while preserving full-parameter adaptability. \ding{186} \textbf{Hybrid PEFT} methods integrate multiple strategies from different PEFT categories to capitalize on their respective advantages \cite{llm_adapters,hybrid-1,hybrid-2}. For instance, NOAH \cite{hybrid-3} and AUTOPEFT \cite{hybrid-4}, leverage neural architecture search to identify effective PEFT combinations for specific tasks. 
In this paper, we primarily focus on LoRA, a reparameterized PEFT method.

\vspace{-0.5em}
\subsection{LoRA and its Variants}
\vspace{-0.5em}
The core idea of LoRA \cite{lora} is to approximate weight updates using mergeable, low-rank matrix pathways. Its variants can be broadly categorized into several types:
\textbf{\ding{182} Novel Branch Designs} primarily focus on remodeling or reformulating the original low-rank matrix approximation pathway, with notable examples including VeRA \cite{vera}, FourierFT \cite{fourierft}, PiSSA \cite{pissa}, and DoRA \cite{dora}.
\textbf{\ding{183} Multi-Task Variants}, exemplified by MoELoRA \cite{moelora-1}, MoA \cite{moelora-2}, CA-LoRA \cite{calora}, and HydraLoRA \cite{hydralora}, are engineered to enhance cross-task generalization—particularly in scenarios such as multi-task learning, domain adaptation, and continual learning—often through the strategic employment of LoRA module mixtures or ensembles.
\textbf{\ding{184} Linear Variants}, including AdaLoRA \cite{adalora}, SaLoRA \cite{salora}, MoSLoRA \cite{moslora}, and FLoRA \cite{flora}, typically augment the LoRA framework by incorporating an additional linear matrix between the two original low-rank factors, thereby bolstering information capture during the training phase.
Beyond these, several nonlinear LoRA variants have recently emerged, including LoRAN \cite{loran}, SineLoRA \cite{sinelora}, LoDA \cite{loda}, NEAT \cite{neat}, and CoLA \cite{cola-nonlinear}.
However, these recent nonlinear variants do not resolve inherent low-rank bottleneck in LoRA. In contrast, our method pairs nonlinearities with a focus on LoRA's fundamental structural limitations, achieving a superior balance between performance and parameter efficiency.

\vspace{-0.5em}
\section{Conclusion}
\vspace{-0.5em}
In this paper, we revisit LoRA from the perspective of linear mappings and introduce nonlinearity into LoRA by proposing \ourmethod, an MLP-like structure. \ourmethod incorporates an adaptive nonlinear layer that includes both fixed and learnable nonlinearities between the two low-rank matrices. \ourmethod achieves a superior balance between performance and parameters across tasks in both the NLP and CV domains. We hope that \ourmethod will inspire further exploration of nonlinear extensions to LoRA.
\vspace{-1em}
\paragraph{Limitation} A potential limitation is that, due to limited computational resources, we do not evaluate performance on larger pretrained models in this study, leaving this exploration for future work.
\vspace{-1em}
\paragraph{Broader Impact} As a novel nonlinear method, \ourmethod is envisioned for broad future applications in key sectors such as healthcare and finance. 
It is anticipated to deliver more accurate and reliable services while significantly reducing resource consumption, thereby better serving human society.

\begin{ack}
This work is supported by the State Key Laboratory of General Artificial Intelligence; and the National Natural Science Foundation of China (Grant No. 62276006).
\end{ack}

\bibliography{ref}
\bibliographystyle{unsrt}

\section*{NeurIPS Paper Checklist}

The checklist is designed to encourage best practices for responsible machine learning research, addressing issues of reproducibility, transparency, research ethics, and societal impact. Do not remove the checklist: {\bf The papers not including the checklist will be desk rejected.} The checklist should follow the references and follow the (optional) supplemental material.  The checklist does NOT count towards the page
limit. 

Please read the checklist guidelines carefully for information on how to answer these questions. For each question in the checklist:
\begin{itemize}
    \item You should answer \answerYes{}, \answerNo{}, or \answerNA{}.
    \item \answerNA{} means either that the question is Not Applicable for that particular paper or the relevant information is Not Available.
    \item Please provide a short (1–2 sentence) justification right after your answer (even for NA). 
\end{itemize}

{\bf The checklist answers are an integral part of your paper submission.} They are visible to the reviewers, area chairs, senior area chairs, and ethics reviewers. You will be asked to also include it (after eventual revisions) with the final version of your paper, and its final version will be published with the paper.

The reviewers of your paper will be asked to use the checklist as one of the factors in their evaluation. While "\answerYes{}" is generally preferable to "\answerNo{}", it is perfectly acceptable to answer "\answerNo{}" provided a proper justification is given (e.g., "error bars are not reported because it would be too computationally expensive" or "we were unable to find the license for the dataset we used"). In general, answering "\answerNo{}" or "\answerNA{}" is not grounds for rejection. While the questions are phrased in a binary way, we acknowledge that the true answer is often more nuanced, so please just use your best judgment and write a justification to elaborate. All supporting evidence can appear either in the main paper or the supplemental material, provided in appendix. If you answer \answerYes{} to a question, in the justification please point to the section(s) where related material for the question can be found.

IMPORTANT, please:
\begin{itemize}
    \item {\bf Delete this instruction block, but keep the section heading ``NeurIPS Paper Checklist"},
    \item  {\bf Keep the checklist subsection headings, questions/answers and guidelines below.}
    \item {\bf Do not modify the questions and only use the provided macros for your answers}.
\end{itemize}


\begin{enumerate}

\item {\bf Claims}
    \item[] Question: Do the main claims made in the abstract and introduction accurately reflect the paper's contributions and scope?
    \item[] Answer: \answerYes{} 
    \item[] Justification: In this paper, we introduce a novel non-linear parameter-efficient fine-tuning method and we claim the contributions and scope in the abstract and introduction sections (See Abstract and Introduction Section).

\item {\bf Limitations}
    \item[] Question: Does the paper discuss the limitations of the work performed by the authors?
    \item[] Answer: \answerYes{} 
    \item[] Justification: In this work, we systematically discuss the limitations of our research and outline directions for future work (See Conclusion Section).

\item {\bf Theory assumptions and proofs}
    \item[] Question: For each theoretical result, does the paper provide the full set of assumptions and a complete (and correct) proof?
    \item[] Answer: \answerYes{} 
    \item[] Justification: In this work, we analyze our proposed method from a theoretical perspective and provide complete and detailed proofs (See Method Section and Appendix).

    \item {\bf Experimental result reproducibility}
    \item[] Question: Does the paper fully disclose all the information needed to reproduce the main experimental results of the paper to the extent that it affects the main claims and/or conclusions of the paper (regardless of whether the code and data are provided or not)?
    \item[] Answer: \answerYes{} 
    \item[] Justification: We provide the code necessary for replicating the studies described in this paper via an anonymous link, and we detail the experimental setup for the replication in the article itself (See Appendix).

\item {\bf Open access to data and code}
    \item[] Question: Does the paper provide open access to the data and code, with sufficient instructions to faithfully reproduce the main experimental results, as described in supplemental material?
    \item[] Answer: \answerYes{} 
    \item[] Justification: For the datasets disclosed in the article, we have provided information regarding their sources and origins (See Appendix).

\item {\bf Experimental setting/details}
    \item[] Question: Does the paper specify all the training and test details (e.g., data splits, hyperparameters, how they were chosen, type of optimizer, etc.) necessary to understand the results?
    \item[] Answer: \answerYes{} 
    \item[] Justification: We have specified all the training and test details (e.g., data splits, hyperparameters, how they were chosen, type of optimizer, etc.) necessary to understand the results (See Appendix).

\item {\bf Experiment statistical significance}
    \item[] Question: Does the paper report error bars suitably and correctly defined or other appropriate information about the statistical significance of the experiments?
    \item[] Answer: \answerYes{} 
    \item[] Justification: Experimental results are tested multiple times to ensure stability and reliability (See Experiments).

\item {\bf Experiments compute resources}
    \item[] Question: For each experiment, does the paper provide sufficient information on the computer resources (type of compute workers, memory, time of execution) needed to reproduce the experiments?
    \item[] Answer: \answerYes{} 
    \item[] Justification: In this paper, we provide detailed information about the experimental resources, including GPU configurations used in our studies (See Appendix). 
    
\item {\bf Code of ethics}
    \item[] Question: Does the research conducted in the paper conform, in every respect, with the NeurIPS Code of Ethics \url{https://neurips.cc/public/EthicsGuidelines}?
    \item[] Answer: \answerYes{} 
    \item[] Justification: The study presented in this paper conforms to the NeurIPS Code of Ethics.

\item {\bf Broader impacts}
    \item[] Question: Does the paper discuss both potential positive societal impacts and negative societal impacts of the work performed?
    \item[] Answer: \answerYes{} 
    \item[] Justification: We have provided the societal impacts of the work (See Conclusion).
    
\item {\bf Safeguards}
    \item[] Question: Does the paper describe safeguards that have been put in place for responsible release of data or models that have a high risk for misuse (e.g., pretrained language models, image generators, or scraped datasets)?
    \item[] Answer: \answerNA{} 
    \item[] Justification: This paper does not address issues related to this aspect.

\item {\bf Licenses for existing assets}
    \item[] Question: Are the creators or original owners of assets (e.g., code, data, models), used in the paper, properly credited and are the license and terms of use explicitly mentioned and properly respected?
    \item[] Answer: \answerYes{} 
    \item[] Justification: All creators and original owners of the assets used in our paper, such as code, data, and models, have been properly credited. We have explicitly mentioned the licenses and terms of use for each asset and have ensured full compliance with these terms throughout our research.

\item {\bf New assets}
    \item[] Question: Are new assets introduced in the paper well documented and is the documentation provided alongside the assets?
    \item[] Answer: \answerNA{} 
    \item[] Justification: The research presented in this paper is not concerned with new assets.

\item {\bf Crowdsourcing and research with human subjects}
    \item[] Question: For crowdsourcing experiments and research with human subjects, does the paper include the full text of instructions given to participants and screenshots, if applicable, as well as details about compensation (if any)? 
    \item[] Answer: \answerNA{} 
    \item[] Justification: This paper does not involve experiments or research related to human subjects.

\item {\bf Institutional review board (IRB) approvals or equivalent for research with human subjects}
    \item[] Question: Does the paper describe potential risks incurred by study participants, whether such risks were disclosed to the subjects, and whether Institutional Review Board (IRB) approvals (or an equivalent approval/review based on the requirements of your country or institution) were obtained?
    \item[] Answer: \answerNA{} 
    \item[] Justification: This paper does not address potential risks incurred by study participants.

\item {\bf Declaration of LLM usage}
    \item[] Question: Does the paper describe the usage of LLMs if it is an important, original, or non-standard component of the core methods in this research? Note that if the LLM is used only for writing, editing, or formatting purposes and does not impact the core methodology, scientific rigorousness, or originality of the research, declaration is not required.
    \item[] Answer: \answerNA{} 
    \item[] Justification: The core method development in this research does not involve LLMs as any important, original, or non-standard components.

\end{enumerate}
\newpage
\appendix

\section{Notations}
We summarize the notations used throughout the manuscript in Table \ref{tab:notation}.
\begin{table}[!htbp]\footnotesize
  \centering
  \caption{Notations commonly used in the \ourmethod method.}
  \setlength{\tabcolsep}{4pt}
  \renewcommand\arraystretch{1.3}
  \vspace{0.5em}
    \begin{tabular}{ll}
    \toprule
    Notation & Definition \\
    \midrule
    $\mathcal{W}_0 \in \mathbb{R}^{d_{\mathrm{out}}\times d_{\mathrm{in}}}$ 
      & Pretrained weight matrix \\
    $\Delta\mathcal{W} \in \mathbb{R}^{d_{\mathrm{out}}\times d_{\mathrm{in}}}$ 
      & Weight update matrix \\
    $\mathbf{A}\in \mathbb{R}^{\widetilde r \times d_{\mathrm{in}}}$ 
      & Downward projector (low‐rank matrix) \\
    $\mathbf{B}\in \mathbb{R}^{d_{\mathrm{out}}\times \widetilde r}$ 
      & Upward projector (low‐rank matrix) \\
    $r$ & Original LoRA rank ($r \ll \min(d_{\mathrm{in}},d_{\mathrm{out}})$) \\
    $\widetilde r$ & Compressed hidden dimension ($\widetilde r \ll r$) \\
    $\mathbf{x}\in \mathbb{R}^{d_{\mathrm{in}}}$ 
      & Input vector \\
    $\mathbf{h}\in \mathbb{R}^{d_{\mathrm{out}}}$ 
      & Output vector \\
    $\sigma(\cdot)$ & Adaptive Nonlinear Layer (ANL) mapping $\mathbb{R}^{\widetilde r}\!\to\!\mathbb{R}^{\widetilde r}$ \\
    $\mathcal{P}_{\mathrm{down}}$ & Projection onto $\widetilde r$‐dim hidden space (matrix $\mathbf{A}$) \\
    $\mathcal{P}_{\mathrm{self}}$ & Self‐projection in hidden space (matrix $\mathbf{H}\in\mathbb{R}^{\widetilde r\times \widetilde r}$) \\
    $\mathcal{P}_{\mathrm{up}}$   & Projection back to $d_{\mathrm{out}}$ (matrix $\mathbf{B}$) \\
    $\mathcal{F}(\cdot)$ & Fixed nonlinearity (e.g.\ $\tanh$) \\
    $\mathcal{L}(\cdot)$ & Learnable nonlinearity (B‐spline based) \\
    $\mathbf{w}_s\in \mathbb{R}^{\widetilde r}$ 
      & Spline weight vector in $\mathcal{L}$ \\
    $\mathbf{s}(\mathbf{Z})$ 
      & Spline basis functions applied to each component of $\mathbf{Z}$ \\
    $T$ & Number of training epochs \\
    \bottomrule
    \end{tabular}
  \label{tab:notation}
\end{table}

\section{Complete Process}
\label{app-complete}
In this section, we first provide further details on both the static weight merging operation performed after the training phase and the actual process at inference time. We then offer empirical validation for our approach. Finally, the complete algorithm workflow of our \ourmethod is presented in Algo. \ref{alg:aurora}.
\subsection{Weights Merging \& Inference Phase} 
\label{app-inference}
During the training phase, to enhance training flexibility, \ourmethod utilizes a \textit{dynamic}, input-dependent update mechanism formulated as $\mathbf{B}  \cdot \sigma(\mathbf{A} \mathbf{x})$.
After the training phase, once all learnable parameters are fixed and considered optimized, \ourmethod transitions to a \textit{static} form for the inference stage. This static form, given by $\Delta \mathcal{W} = \mathbf{B}  \cdot \sigma(\mathbf{A})$, facilitates the seamless integration of \ourmethod into the pre-trained weight $\mathcal{W}_0$, consistent with standard LoRA.
Therefore, after the weights are merged, the effective forward propagation process at inference time is formally given by:
\begin{equation}\label{eq:inference}
    \mathbf{h} = \mathcal{W}_0 \mathbf{x} + \Delta \mathcal{W} \mathbf{x} = \mathcal{W}_0 \mathbf{x} +\mathbf{B} \cdot \sigma(\mathbf{A}) \mathbf{x}.
\end{equation}
Such an approach eliminates any additional computational overhead during inference, while concurrently preserving superior performance.
\paragraph{Why is This Strategy Effective?}
To validate our strategy of \textit{dynamic} training combined with \textit{static} inference, we empirically compare it against both \textit{fully dynamic} and \textit{fully static} approaches.
For this purpose, we introduce two comparative variants: (1) \ourmethod-$\mathcal{D}$, which maintains dynamic processing throughout both training and inference (i.e., its forward pass is consistently $\mathbf{h} = \mathcal{W}_0 \mathbf{x} + \mathbf{B} \cdot \sigma(\mathbf{A} \mathbf{x})$), and (2) \ourmethod-$\mathcal{S}$, which consistently employs a static form for both phases (i.e., $\mathbf{h} = \mathcal{W}_0 \mathbf{x} + \mathbf{B} \cdot \sigma(\mathbf{A}) \mathbf{x}$).
We evaluate \ourmethod, \ourmethod-$\mathcal{D}$, and \ourmethod-$\mathcal{S}$ by fine-tuning LLaMA3-8B for Commonsense Reasoning on datasets including \textsc{arc-e}, \textsc{obqa}, \textsc{siqa}, and \textsc{arc-c}, and record both accuracy and total training and inference time.
From Table \ref{tab:study-inference}, we observe that: \ding{182} \ourmethod exhibits nearly identical performance to \ourmethod-$\mathcal{D}$ and significantly outperforms \ourmethod-$\mathcal{S}$ in practice; \ding{183} \ourmethod achieves a runtime comparable to that of \ourmethod-$\mathcal{S}$ and is markedly faster than \ourmethod-$\mathcal{D}$.
Therefore, our practical implementation adopts this dynamic training with static inference strategy, which can be seamlessly merged into pre-trained weights after the training phase (consistent with standard LoRA), and thereby also achieves an effective trade-off between performance and computational cost.

\begin{table*}[!t]
\centering
\caption{Comparison of accuracy and total time consumption for different settings on eight Commonsense Reasoning datasets, using LLaMA3-8B as pre-trained model.}
\vspace{0.5em}
\label{tab:study-inference}
\renewcommand\tabcolsep{5.3pt}
\renewcommand\arraystretch{1.1}

\resizebox{\linewidth}{!}{
\begin{tabular}{l|c|ccccccccc}
\Xhline{1.2pt}
\rowcolor{CadetBlue!20} 
\textbf{Setting} & \textbf{Time} &  \textbf{BoolQ} & \textbf{PIQA} & \textbf{SIQA} & \textbf{HellaSwag} & \textbf{WinoGrande} & \textbf{ARC-e} & \textbf{ARC-c} & \textbf{OBQA} & \textbf{Avg.} \\
\Xhline{1.2pt}


LoRA & $15.05$ h & 70.8 & 85.2 & 79.9 & 91.7 & 84.3 & 84.2 & 71.2 & 79.0 & 80.8 \\

\cellcolor{gray!10}\ourmethod-$\mathcal{S}$ & \cellcolor{gray!10}$15.18$ h & \cellcolor{gray!10}71.4 & \cellcolor{gray!10}86.9 & \cellcolor{gray!10}78.1 & \cellcolor{gray!10}93.5 & \cellcolor{gray!10}81.7 & \cellcolor{gray!10}88.5 & \cellcolor{gray!10}78.1 & \cellcolor{gray!10}83.9 & \cellcolor{gray!10}82.8 \\

\ourmethod-$\mathcal{D}$ & $15.60$ h & 72.6 & 87.5 & 79.2 & 94.3 & 83.0 & 89.5 & 78.9 & 85.0 & 83.8 \\

\cellcolor{gray!10}\ourmethod               & \cellcolor{gray!10}$15.28$ h & \cellcolor{gray!10}72.5 & \cellcolor{gray!10}87.4 & \cellcolor{gray!10}79.0 & \cellcolor{gray!10}94.2 & \cellcolor{gray!10}83.0 & \cellcolor{gray!10}89.3 & \cellcolor{gray!10}78.8 & \cellcolor{gray!10}84.8 & \cellcolor{gray!10}83.6 \\
\Xhline{1.2pt}
\end{tabular}
}
\end{table*}

\subsection{Algorithm Workflow} 
The algorithm framework is presented in Algo.~\ref{alg:aurora}.

\begin{algorithm}[!htpb]
\caption{Algorithm workflow of \ourmethod}\label{alg:aurora}
\Input{Pretrained weight $\mathcal{W}_0$, low-rank factors $\mathbf{A}\in\mathbb{R}^{\tilde r\times d}$ and $\mathbf{B}\in\mathbb{R}^{d\times\tilde r}$, ANL parameters, training data $\{(\mathbf{x}_i,y_i)\}_{i=1}^N$, number of epochs $T$}
\BlankLine
\tcc{\textcolor{blue}{\textbf{Training Phase} (dynamic update: $\mathbf{B}\,\sigma(\mathbf{A}\mathbf{x})$)}}
\For{\rm epoch $t \leftarrow 1$ \KwTo $T$}{
  \For{each minibatch $\{\mathbf{x},y\}$ in training data}{
    \tcc{\textcolor{blue}{Forward pass with ANL on input}}
    $\mathbf{h} \leftarrow \mathcal{W}_0\,\mathbf{x} \;+\;\mathbf{B}\cdot\sigma\bigl(\mathbf{A}\,\mathbf{x}\bigr)$ \Comment*[r]{\textcolor{blue}{Eq.~\ref{eq:train-forward}}}
    Compute loss $\mathcal{L}(\mathbf{h},y)$
    
    \tcc{\textcolor{blue}{Backpropagate through $\mathbf{A},\mathbf{B},\,$and ANL parameters}}
    Backpropagate and update $\{\mathbf{A},\mathbf{B},\text{ANL}\}$
  }
}
\BlankLine
\tcc{\textcolor{blue}{\textbf{Inference Preparation} (static merge: $\mathbf{B}\,\sigma(\mathbf{A})$)}}
\Fn{\texttt{MergeWeights}()}{
  \tcc{\textcolor{blue}{Compute element-wise ANL on matrix $\mathbf{A}$}}
  $\quad \widetilde{\mathbf{A}} \leftarrow \sigma(\mathbf{A})$
  
  \tcc{\textcolor{blue}{Form the effective weight update}}
  $\quad \Delta \mathcal{W} \leftarrow \mathbf{B}\,\widetilde{\mathbf{A}}$ \Comment*[r]{\textcolor{blue}{Eq.~\ref{eq:delta}}}
  \tcc{\textcolor{blue}{Merge into pretrained weights}}
  $\quad \mathcal{W} \leftarrow \mathcal{W}_0 + \Delta \mathcal{W}$
  
  \Return $\mathcal{W}$
}
\BlankLine
\tcc{\textcolor{blue}{\textbf{Inference Phase} (static forward: no extra ANL)}}
\For{each test sample $\mathbf{x}$}{
  $\quad \mathcal{W} \leftarrow \texttt{MergeWeights}()$
  
  $\quad \mathbf{h} \leftarrow \mathcal{W}\,\mathbf{x}$ \Comment*[r]{\textcolor{blue}{Eq.~\ref{eq:inference}}}
  \tcc{\textcolor{blue}{Use $\mathbf{h}$ for downstream prediction}}
}
\end{algorithm}

\section{Intuitive Case}
\label{app-case}
To intuitively and concisely illustrate the impact of nonlinear mapping on the matrices, we first consider a simple scenario where \( \mathbf{A} \in \mathbb{R}^{1 \times 2} \) and \( \mathbf{B} \in \mathbb{R}^{2 \times 1} \):
\begin{equation}
\mathbf{A} = \begin{bmatrix} a_1 & a_2 \end{bmatrix}, \quad \mathbf{B} = \begin{bmatrix} b_1 \\ b_2 \end{bmatrix}
\end{equation}
We then introduce the LeakyReLU activation function as the nonlinear mapping between the two low-rank matrices. Depending on the elements of matrix \( \mathbf{A} \), the resulting weight update comprises the following four matrix structures:
\begin{equation}
\begin{aligned}
\Delta \mathcal{W} &=
\begin{bmatrix} 
b_1 a_1 & b_2 a_1 \\ 
b_1 a_2 & b_2 a_2 
\end{bmatrix}, 
&& \text{if } a_1 > 0 \text{ and } a_2 > 0, \\[10pt]
&\begin{bmatrix} 
b_1 (\alpha a_1) & b_2 (\alpha a_1) \\ 
b_1 a_2 & b_2 a_2 
\end{bmatrix}, 
&& \text{if } a_1 \leq 0 \text{ and } a_2 > 0, \\[10pt]
&\begin{bmatrix} 
b_1 a_1 & b_2 a_1 \\ 
b_1 (\alpha a_2) & b_2 (\alpha a_2) 
\end{bmatrix}, 
&& \text{if } a_1 > 0 \text{ and } a_2 \leq 0, \\[10pt]
&\begin{bmatrix} 
b_1 (\alpha a_1) & b_2 (\alpha a_1) \\ 
b_1 (\alpha a_2) & b_2 (\alpha a_2) 
\end{bmatrix}, 
&& \text{if } a_1 \leq 0 \text{ and } a_2 \leq 0.
\end{aligned}
\end{equation}
where $\alpha$ is a hyperparameter of LeakyReLU, usually a positive number less than 1. Meanwhile, LoRA can only produce:
\begin{equation}
    \Delta \mathcal{W} = \begin{bmatrix} b_1 a_1 & b_2 a_1 \\ b_1 a_2 & b_2 a_2 \end{bmatrix}
\end{equation}
Under the LeakyReLU activation function, each negative component of $\mathbf{A}$ is scaled by the factor $\alpha$, while positive components remain unchanged. This piecewise linear mapping disrupts the uniformity of low-rank multiplication, causing the final weight update $\Delta \mathcal{W}$ to depend not only on the product of $\mathbf{A}$ and $\mathbf{B}$ but also on the local behavior of each element in $\mathbf{A}$. Consequently, even under tight rank constraints, the model benefits from a richer set of possible weight updates, enhancing its adaptability to varying input distributions.

\section{Proof of Proposition \ref{prop:LAE}}
\label{app-proof-LAE}

In this appendix, we provide the complete theoretical analysis and proof of Proposition \ref{prop:LAE}. 
We first restate the problem setup and then present the necessary lemmas, followed by the main proof.

\begin{definition}[Best Linear Rank-$r$ Error]
For $M\in\mathbb{R}^{d_{\mathrm{out}}\times d_{\mathrm{in}}}$, define
\[
  \varepsilon_r(M)
  ~:=~
  \inf_{
    U\in \mathbb{R}^{d_{\mathrm{out}}\times r},\,
    V\in \mathbb{R}^{r\times d_{\mathrm{in}}}
  }
  \|\,M - U\,V\|.
\]
If $\mathrm{rank}(M)>r$, then $\varepsilon_r(M)>0$ \cite{topicsma}. 
\end{definition}

\begin{assumption}[Bounded Input Domain]
\label{assump:bounded-domain}
We assume $\mathbf{x}\in \mathcal{X}\subset \mathbb{R}^{d_{\mathrm{in}}}$ satisfies $\|\mathbf{x}\|\le X_{\max}$. 
Then for $A\in \mathbb{R}^{r\times d_{\mathrm{in}}}$ with $\|A\|\le A_{\max}$, 
the vector $\mathbf{z}=A\mathbf{x}$ remains in a bounded set $\Omega \subset \mathbb{R}^r$ (compact).
\end{assumption}

\begin{definition}[Nonlinear Low-Rank Update]
\label{def:nonlinear-update}
Let
\[
  M_{\mathrm{nonlinear}}(\mathbf{x})
  ~=~
  B\,\sigma\bigl(A\,\mathbf{x}\bigr),
\]
with $A\in \mathbb{R}^{r\times d_{\mathrm{in}}}, B\in\mathbb{R}^{d_{\mathrm{out}}\times r}$. 
The map $\sigma:\mathbb{R}^r\to\mathbb{R}^r$ is given by
\[
  \sigma(\mathbf{z})
  ~=~
  \mathcal{F}(\mathbf{z}) 
  ~+~
  \mathbf{w}_s\,\cdot\,\mathbf{s}(\mathbf{z}),
\]
where $\mathcal{F}$ is a fixed bounded function (e.g.\ $\tanh$-based) and $\mathbf{s}(\mathbf{z})$ denotes B-spline basis functions in $\mathbb{R}^r$.
\end{definition}

\begin{lemma}[Piecewise Polynomial Approximation]
\label{lem:b-spline-approx}
Consider $f:\Omega\to\mathbb{R}^m$ with $f\in C^k(\Omega)$ on a bounded domain $\Omega\subset \mathbb{R}^r$. 
Let $\Delta>0$ be the subdivision size in each coordinate axis for constructing a tensor-product B-spline. 
Then there exists a B-spline $g(\mathbf{z})$ such that
\[
  \sup_{\mathbf{z}\in\Omega}
  \|f(\mathbf{z}) - g(\mathbf{z})\|
  ~\le~
  C_f(\Delta)^k,
\]
where $C_f>0$ is a constant depending on $f$'s $k$-th order partial derivatives and the geometry of $\Omega$ \cite{spline1}.
\end{lemma}

\begin{lemma}[Combining Fixed and Learnable Nonlinearities]
\label{lem:comb-F-S}
Let $\Omega\subset \mathbb{R}^r$ be compact. 
Assume $\mathcal{F}:\Omega\to \mathbb{R}^r$ is fixed, bounded, and $C^1$, and let $h(\mathbf{z})\in C^k(\Omega)$ be the target. 
Define
\[
  \sigma(\mathbf{z})
  ~=~
  \mathcal{F}(\mathbf{z})
  ~+~
  \mathbf{w}_s\cdot \mathbf{s}(\mathbf{z}),
\]
where $\mathbf{s}(\mathbf{z})$ is a B-spline basis. 
Then, for any $\epsilon>0$, one can choose $\Delta>0$ and $\mathbf{w}_s$ such that
\[
  \sup_{\mathbf{z}\in\Omega}
  \bigl\|
    h(\mathbf{z})
    -
    \bigl[\mathcal{F}(\mathbf{z}) + \mathbf{w}_s\cdot \mathbf{s}(\mathbf{z})\bigr]
  \bigr\|
  ~\le~\epsilon.
\]
Furthermore, the error decays like $O\bigl((\Delta)^k\bigr)$ as $\Delta\to 0$.
\end{lemma}

\begin{proof}[Proof of Lemma~\ref{lem:comb-F-S}]
Let $r(\mathbf{z})=h(\mathbf{z})-\mathcal{F}(\mathbf{z})$. 
Since $h\in C^k(\Omega)$ and $\mathcal{F}$ is fixed and $C^1$, $r(\mathbf{z})$ remains $C^k$. 
Applying Lemma~\ref{lem:b-spline-approx} to $r(\mathbf{z})$ yields a B-spline $g(\mathbf{z})$ with $\|r(\mathbf{z})-g(\mathbf{z})\|\le C\,(\Delta)^k$. 
Hence $\|h(\mathbf{z}) - [\mathcal{F}(\mathbf{z}) + g(\mathbf{z})]\|\le C\,(\Delta)^k$, completing the proof.
\end{proof}

We restate Proposition \ref{prop:LAE} here for completeness:

\textbf{Proposition \ref{prop:LAE}}
(Lower Approximation Error)
\emph{Let $M\in\mathbb{R}^{d_{\mathrm{out}}\times d_{\mathrm{in}}}$ with $\mathrm{rank}(M)>r$. Then 
\[
  \varepsilon_r(M)
  ~:=~
  \inf_{U,V}\|M - U\,V\|
  \;>\;0.
\]
Under Definition~\ref{def:nonlinear-update}, there exist $A^*, B^*$ and a B-spline parameter set $(\mathbf{w}_s^*)$ such that
\[
  \|\,M - M_{\mathrm{nonlinear}}\|
  ~\le~
  c\,\varepsilon_r(M),
  \quad
  0<c<1.
\]}

\begin{proof}[Proof]
Let $M^*=U^*V^*$ be the best linear rank-$r$ approximation of $M$, so $\|M - M^*\|=\varepsilon_r(M)$. 
Denote the residual $R=M - M^*$, and we have $\|R\|=\varepsilon_r(M)$. 

By Assumption~\ref{assump:bounded-domain}, for $\|\mathbf{x}\|\le X_{\max}$, let $\mathbf{z}=A^*\mathbf{x}\in\Omega\subset\mathbb{R}^r$, with $\|A^*\|\le A_{\max}$. 
Thus $\mathbf{z}$ lies in a compact $\Omega$. 
Consider $R(\mathbf{x})$ as a function $h(\mathbf{z})=R(\mathbf{x})$. 
Since $R$ is linear (hence $C^\infty$), $h(\mathbf{z})$ is at least $C^1$ in $\mathbf{z}$. 

From Lemma~\ref{lem:comb-F-S}, there is a B-spline $\mathbf{w}_s^*\cdot \mathbf{s}(\mathbf{z})$ approximating $h(\mathbf{z})-\mathcal{F}(\mathbf{z})$ within $\gamma\,\|R\|$ for some $0<\gamma<1$. 
Define
\[
  \widehat{M}(\mathbf{x})
  :=
  M^*(\mathbf{x})
  ~+~
  B^*\Bigl[\mathcal{F}(A^*\mathbf{x}) 
           + \mathbf{w}_s^*\cdot \mathbf{s}(A^*\mathbf{x})\Bigr].
\]
Then
\[
  \|\widehat{M}-M\|
  ~=~
  \|\bigl(M^* + B^*[\dots]\bigr) - (M^* + R)\|
  ~=~
  \bigl\|\,B^*\bigl[\mathcal{F}(\cdot) + \mathbf{w}_s^*\cdot \mathbf{s}(\cdot)\bigr] - R\bigr\|
  ~\le~
  \gamma\,\|R\|
  =
  \gamma\,\varepsilon_r(M).
\]
Since $\gamma<1$, we obtain $\|\widehat{M}-M\|<\varepsilon_r(M)$, which strictly improves upon the LoRA limit. 
Setting $\widehat{M}\equiv M_{\mathrm{nonlinear}}$ completes the proof.
\end{proof}

\section{Proof of Proposition \ref{prop:gradient-boundedness}}
\label{proof: prop-3}
\textbf{Proposition \ref{prop:gradient-boundedness}.}
\emph{In the \ourmethod, the use of the $\tanh$ activation function and B-spline basis functions results in bounded gradients with respect to both the inputs and the model parameters.}

The loss function $L$ for a single data point $(x, y)$ is defined as $L(x, y) = \frac{1}{2} \left\| f_{\text{\ourmethod}}(x) - y \right\|^2$, where $y \in \mathbb{R}^{d_{\text{out}}}$ is the target output. We will compute and bound the gradients of the loss function with respect to $W_b$, $w_s$ and the input $x$.

\begin{lemma}
The gradients of the loss function with respect to $W_b$ is bounded.
\end{lemma}

\begin{proof}
Compute the gradient $\frac{\partial L}{\partial W_b}$:
\[
\frac{\partial L}{\partial W_b} = \left( f_{\text{\ourmethod}}(x) - y \right)^\top B \cdot \frac{\partial \text{ANL}(A^\top x)}{\partial W_b}.
\]
Compute $\frac{\partial \text{ANL}(z)}{\partial W_b}$:
\[
\frac{\partial \text{ANL}(z)}{\partial W_b} = \frac{\partial \phi(W_b \phi(z))}{\partial W_b} = \operatorname{diag}\left( \phi'\left( W_b \phi(z) \right) \right) \cdot \phi(z)^\top,
\]
where $\phi'(u) = 1 - \tanh^2(u)$ is the derivative of Tanh, $\operatorname{diag}(v)$ denotes a diagonal matrix with vector $v$ on the diagonal. It is not difficult to deduce that $\phi(z) \in (-1, 1)$ since Tanh outputs are bounded, and $\phi'(u) \in (0, 1]$ because $1 - \tanh^2(u) \leq 1$. So $\frac{\partial \text{ANL}(z)}{\partial W_b}$ is bounded. Consequently, $\frac{\partial L}{\partial W_b}$ is bounded as it is a product of bounded terms.
\end{proof}

\begin{lemma}
    The gradients of the loss function with respect to $w_s$ is bounded.
\end{lemma}

\begin{proof}
    Compute the gradient $\frac{\partial L}{\partial w_s}$:
\[
\frac{\partial L}{\partial w_s} = \left( f_{\text{\ourmethod}}(x) - y \right)^\top B \cdot \frac{\partial \text{ANL}(A^\top x)}{\partial w_s}.
\]
Compute $\frac{\partial \text{ANL}(z)}{\partial w_s}$:
\[
\frac{\partial \text{ANL}(z)}{\partial w_s} = s(z),
\]
since $\text{ANL}(z)$ is linear in $w_s$. B-spline basis functions $B(z_i)$ are smooth and have compact support, and the outputs of $B(z_i)$ are bounded. Therefore, $s(z)$ is bounded, and thus $\frac{\partial L}{\partial w_s}$ is bounded.
\end{proof}

\begin{lemma}
    The gradients of the loss function with respect to the input $x$ is bounded.
\end{lemma}

\begin{proof}
    Compute the gradient $\frac{\partial L}{\partial x}$:
\[
\frac{\partial L}{\partial x} = \left( f_{\text{\ourmethod}}(x) - y \right)^\top \left( W + B \cdot \frac{\partial \text{ANL}(A^\top x)}{\partial x} \right).
\]
Compute $\frac{\partial \text{ANL}(A^\top x)}{\partial x}$:
\[
\frac{\partial \text{ANL}(A^\top x)}{\partial x} = \frac{\partial \text{ANL}(z)}{\partial z} \cdot A^\top,
\]
where $z = A^\top x$. Compute $\frac{\partial \text{ANL}(z)}{\partial z}$:
\[
\frac{\partial \text{ANL}(z)}{\partial z} = \frac{\partial \phi(W_b \phi(z))}{\partial z} + w_s \cdot \frac{\partial s(z)}{\partial z}.
\]
Compute $\frac{\partial \phi(W_b \phi(z))}{\partial z}$:
\[
\frac{\partial \phi(W_b \phi(z))}{\partial z} = \operatorname{diag}\left( \phi'\left( W_b \phi(z) \right) \right) W_b \operatorname{diag}\left( \phi'(z) \right).
\]
$\phi'(z)$ and $\phi'\left( W_b \phi(z) \right)$ are bounded in $(0, 1]$. Entries of $W_b$ are finite. So $\frac{\partial \phi(W_b \phi(z))}{\partial z}$ is bounded.

Compute $\frac{\partial s(z)}{\partial z}$:
\[
\frac{\partial s(z)}{\partial z} = \left[ B'(z_1), B'(z_2), \dots, B'(z_r) \right]^\top.
\]
Derivatives $B'(z_i)$ of B-spline functions are bounded due to their polynomial nature and compact support. So $\frac{\partial s(z)}{\partial z}$ is bounded. Therefore, $\frac{\partial \text{ANL}(z)}{\partial z}$ is bounded, leading to $\frac{\partial \text{ANL}(A^\top x)}{\partial x}$ being bounded. Consequently, $\frac{\partial L}{\partial x}$ is bounded.
\end{proof}

\section{Dataset}
\subsection{GLUE Benchmark}
\label{appendix-dataset-glue}
The GLUE (General Language Understanding Evaluation), as introduced in \cite{glue}, is a widely adopted benchmark in the field of Natural Language Processing (NLP). GLUE encompasses a collection of eight diverse NLP tasks: MNLI (natural language inference), SST-2 (sentiment analysis), MRPC (paraphrase detection), CoLA (linguistic acceptability), QNLI (natural language inference), QQP (question answering), RTE (recognizing textual entailment), and STS-B (textual similarity). The statistical details of these datasets are summarized in Table \ref{tab:dataset-glue}.

\begin{table*}[h!]
\centering
\caption{Detailed task descriptions and dataset statistics for the GLUE benchmark. STS-B is categorized as a regression task, while all other tasks involve single-sentence or sentence-pair classification.}
\label{tab:dataset-glue}
\resizebox{0.9\textwidth}{!}{%
\begin{tabular}{@{}lllcrrrl@{}}
\toprule
\multicolumn{1}{l|}{\textbf{Corpus}} & Task & Metrics & \# Labels & \# Train & \# Val & \# Test  & Domain \\ \midrule
\multicolumn{8}{c}{Single-Sentence Tasks} \\ \midrule
\multicolumn{1}{l|}{CoLA} & Acceptability & Matthews Corr. & 2 & 8.55k & 1.04k & 1.06k & misc. \\
\multicolumn{1}{l|}{SST-2} & Sentiment & Accuracy & 2  & 67.3k & 872 & 1.82k & Movie reviews \\ \midrule
\multicolumn{8}{c}{Similarity and Paraphrase Tasks} \\ \midrule
\multicolumn{1}{l|}{MRPC} & Paraphrase & Accuracy/F1 & 2 & 3.67 & 408 & 1.73k & News \\
\multicolumn{1}{l|}{STS-B} & Sentence similarity & Pearson/Spearman Corr. & 1 & 5.75k & 1.5k & 1.38k & misc. \\
\multicolumn{1}{l|}{QQP} & Paraphrase & Accuracy/F1 & 2  & 364k & 40.4k & 391k & Social QA \\ \midrule
\multicolumn{8}{c}{Inference Tasks} \\ \midrule
\multicolumn{1}{l|}{MNLI} & NLI & Accuracy & 3  & 393k & 19.65k & 19.65k & misc. \\
\multicolumn{1}{l|}{QNLI} & QA/NLI & Accuracy & 2  & 105k & 5.46k & 5.46k & Wikipedia \\
\multicolumn{1}{l|}{RTE} & NLI & Accuracy & 2  & 2.49k & 277 & 3k & News \& Wikipedia \\ \bottomrule
\end{tabular}%
}
\end{table*}

\subsection{Commonsense Reasoning}
\label{appendix-dataset-cr}
Following \cite{llm_adapters}, we use eight datasets in Commonsense Reasoning task. (1) The BoolQ \cite{boolq} dataset is a question-answering benchmark consisting of 15,942 examples, where the questions are naturally occurring and generated in unprompted and unconstrained settings, requiring yes/no answers. (2) The PIQA 
\cite{piqa} dataset presents questions with two potential solutions, demanding physical commonsense reasoning to identify the correct answer. (3) The SIQA \cite{siqa} dataset focuses on reasoning about human actions and their social implications. (4) The HellaSwag \cite{hellas} dataset is designed for commonsense natural language inference (NLI) tasks, where each question includes a context and several potential endings, from which the correct continuation must be selected. (5) The WinoGrande \cite{winog} dataset is a fill-in-the-blank task with binary options, where the goal is to select the most plausible option for a given sentence requiring commonsense reasoning. (6) The ARC-c and (7) ARC-e \cite{arc} datasets refer to the Challenge and Easy sets, respectively, of the ARC dataset, which consists of multiple-choice science questions designed at a grade-school level, with the former being more challenging than the latter. (8) The OBQA \cite{obqa} dataset focuses on questions that necessitate multi-step reasoning, integration of external common knowledge, and in-depth text comprehension. Statistical details are shown in Table \ref{tab:dataset-cr}.

\begin{table}[h!]
\centering
\caption{Details of datasets being evaluated in commonsense reasoning task.}
\label{tab:dataset-cr}
\begin{tabular}{l|llr}
\toprule
Dataset & \#Train & \#Test & Answer \\
\midrule
BoolQ \cite{boolq} & 9.4K & 3270 & Yes/No \\
PIQA \cite{piqa} & 16.1K & 1830 & Option \\
SIQA \cite{siqa} & 33.4K & 1954 & Option\\
HellaSwag \cite{hellas} & 39.9K & 10042 & Option\\
WinoGrande \cite{winog} & 63.2K & 1267 & Option\\
ARC-e \cite{arc} & 1.1K & 2376 & Option\\
ARC-c \cite{arc} & 2.3K & 1172 & Option\\
OBQA \cite{obqa} &  5.0K & 500 & Option\\
\bottomrule
\end{tabular}
\end{table}

\subsection{Image Classification}
\label{appendix-dataset-icl}
We show the details of the datasets in Image Classification task in Table \ref{tab:data-icl}.

\begin{table}[h!]
\centering
\caption{Details of the datasets for the Image Classification task.}
\label{tab:data-icl}
\resizebox{0.7\textwidth}{!}{%
\begin{tabular}{@{}l|clrrc@{}}
\toprule
Dataset & \#Class & \#Train & \#Val & \#Test  & Rescaled resolution \\ \midrule
OxfordPets \cite{pets} & 37  & 3,312 & 368 & 3,669 & \multirow{8}{*}{$224\times224$} \\
StandfordCars \cite{cars} & 196  & 7,329 & 815 & 8,041 &  \\
CIFAR10 \cite{cifar} & 10 & 45,000 & 5,000 & 10,000 &  \\
DTD \cite{dtd} & 47 & 4,060 & 452 & 1,128 &  \\
EuroSAT \cite{eurosat} & 10 & 16,200 & 5,400 & 5,400 &  \\
FGVC \cite{fgvc} & 100 & 3,000 & 334 & 3,333 &  \\
RESISC45 \cite{resisc} & 45 & 18,900 & 6,300 & 6,300 &  \\
CIFAR100 \cite{cifar} & 100 & 45,000 & 5,000 & 10,000 &  \\ \bottomrule
\end{tabular}%
}
\end{table}

\section{Hyperparameters}
To ensure the reproducibility of our experimental results, we provide the detailed hyperparameter settings used in our experiments. In all of our experiments, to achieve a better balance between parameter count and performance, we set the hidden layer dimension (Rank $\widetilde{r}$) of \ourmethod to 2. Correspondingly, we set the hyperparameter $\alpha$ of \ourmethod to 4. Natural Language Understanding and image classification tasks run on four NVIDIA GeForce RTX 4090 (24GB) GPUs.
Commonsense reasoning and subject-driven generation tasks run on NVIDIA L20 (48GB).
\subsection{Natural Language Understanding}
\label{appendix-hyper-nlu}
We provide the hyperparameters used for the GLUE benchmark in natural language understanding experiments in Table \ref{tab:hyper-nlu}. To facilitate reproducibility, we fix the random seed to 0. We tune the learning rate, while all other settings follow those used in LoRA \cite{lora} and FourierFT \cite{fourierft}.
\begin{table}[h!]
\centering
\caption{Hyperparameter setup of \ourmethod for the GLUE benchmark.}
\label{tab:hyper-nlu}
\resizebox{0.8\textwidth}{!}{%
\begin{tabular}{@{}clcccccc@{}}
\toprule
Model & Hyperparameter & \multicolumn{1}{|c}{STS-B} & \multicolumn{1}{c}{RTE} & \multicolumn{1}{c}{MRPC} & \multicolumn{1}{c}{CoLA} & \multicolumn{1}{c}{SST-2} & \multicolumn{1}{c}{QNLI} \\ \midrule
\multirow{5}{*}{Both} & \multicolumn{1}{l|}{Optimizer} & \multicolumn{6}{c}{AdamW} \\
 & \multicolumn{1}{l|}{LR Schedule} & \multicolumn{6}{c}{Linear} \\
 & \multicolumn{1}{l|}{Warmup Ratio} & \multicolumn{6}{c}{0.06} \\
 & \multicolumn{1}{l|}{Rank $\widetilde{r}$} & \multicolumn{6}{c}{2} \\
 & \multicolumn{1}{l|}{$\alpha$} & \multicolumn{6}{c}{4} \\
 \midrule
\multirow{4}{*}{Base} & \multicolumn{1}{l|}{Epochs} & 30 & 80 & 30 & 90 & 30 & 80 \\
 & \multicolumn{1}{l|}{Learning Rate} & 6E-4 & 5E-4 & 8E-4 & 5E-3 & 8E-4 & 5E-3 \\
 & \multicolumn{1}{l|}{Max Seq. Len} & 512 & 512 & 512 & 512 & 512 & 512 \\
 & \multicolumn{1}{l|}{Batch Size} & 64 & 16 & 64 & 32 & 32 & 32 \\ \midrule
\multicolumn{1}{l}{\multirow{4}{*}{Large}} & \multicolumn{1}{l|}{Epochs} & 20 & 10 & 30 & 50 & 40 & 20 \\
\multicolumn{1}{l}{} & \multicolumn{1}{l|}{Learning Rate} & 3E-4 & 4E-4 & 1E-3 & 5E-4 & 8E-4 & 4E-4 \\
\multicolumn{1}{l}{} & \multicolumn{1}{l|}{Max Seq. Len} & 512 & 512 & 512 & 256 & 128 & 512 \\
\multicolumn{1}{l}{} & \multicolumn{1}{l|}{Batch Size} & 16 & 32 & 16 & 16 & 16 & 16 \\ \bottomrule
\end{tabular}%
}
\end{table}

\subsection{Commonsense Reasoning}
\label{appendix-hyper-cr}
We provide the detailed hyperparameters for fine-tuning LLaMA3-8B in the commonsense reasoning task in Table \ref{tab:hyper-cr}.
\begin{table}[h!]
\centering
\caption{Hyperparameter setup of \ourmethod for Commonsense Reasoning.}
\label{tab:hyper-cr}
\begin{tabular}{l|c}
\toprule
Hyperparameter & Commonsense Reasoning \\
\midrule
Rank $\widetilde{r}$ & 2 \\
$\alpha$ & 4 \\
Dropout & 0.05 \\
Batch Size & 16 \\
Optimizer & Adam W \\
Learning Rate & 3e-4 \\
Warmup Steps & 100 \\
Epochs & 3 \\
Target module & q,k,v,up,down \\
\bottomrule
\end{tabular}
\end{table}

\subsection{Image Classification}
\label{appendix-hyper-icl}
We provide the detailed hyperparameters for the image classification in Table \ref{tab:hyper-icl}. We tune the learning rate, while the weight decay value follows the settings used in FourierFT \cite{fourierft} without tuning.
\begin{table}[h!]
\centering
\caption{Hyperparameter setup of \ourmethod for the image classification.}
\label{tab:hyper-icl}
\resizebox{0.8\textwidth}{!}{%
\begin{tabular}{@{}clcccccccc@{}}
\toprule
Model & Hyperparameter & \multicolumn{1}{|c}{OxfordPets} & \multicolumn{1}{c}{StanfordCars} & \multicolumn{1}{c}{CIFAR10} & \multicolumn{1}{c}{DTD} & \multicolumn{1}{c}{EuroSAT} & \multicolumn{1}{c}{FGVC} & \multicolumn{1}{c}{RESISC45} & \multicolumn{1}{c}{CIFAR100} \\ \midrule
\multirow{5}{*}{Both} & \multicolumn{1}{l|}{Optimizer} & \multicolumn{8}{c}{AdamW} \\
 & \multicolumn{1}{l|}{LR Schedule} & \multicolumn{8}{c}{Linear} \\
 & \multicolumn{1}{l|}{Epochs} & \multicolumn{8}{c}{10} \\
 & \multicolumn{1}{l|}{Rank $\widetilde{r}$} & \multicolumn{8}{c}{2} \\
 & \multicolumn{1}{l|}{$\alpha$} & \multicolumn{8}{c}{4} \\
 \midrule
\multirow{3}{*}{Base} & \multicolumn{1}{l|}{Learning Rate (\ourmethod)} & 5e-3 & 1e-2 & 1e-2 & 1e-2 & 5e-3 & 1e-2 & 8e-3 & 8e-3 \\
 & \multicolumn{1}{l|}{Learning Rate (Head)} & 5E-3 & 1e-2 & 3e-2 & 8E-3 & 8E-3 & 1e-2 & 1e-2 & 5e-3\\
 & \multicolumn{1}{l|}{Weight Decay} & 8E-4 & 4E-5 & 9E-5 & 7E-5 & 3E-4 & 7E-5 & 3E-4 & 1E-4 \\ \midrule
\multicolumn{1}{l}{\multirow{3}{*}{Large}} & \multicolumn{1}{l|}{Learning Rate (\ourmethod)} & 5e-3 & 9e-3 & 8e-3 & 8e-3 & 4e-3 & 1.5e-2 & 7.5e-3 & 1.5e-2\\
\multicolumn{1}{l}{} & \multicolumn{1}{l|}{Learning Rate (Head)} & 4e-3 & 8e-3 & 4e-2 & 9e-3 & 8e-3 & 1e-2 & 1.5e-2 & 5e-3\\
 & \multicolumn{1}{l|}{Weight Decay} & 8E-4 & 4E-5 & 9E-5 & 7E-5 & 3E-4 & 7E-5 & 3E-4 & 1E-4 \\ \midrule
\end{tabular}%
}
\end{table}

\section{More Cases of Generated Images}
\label{sec:app-sd}
In Figure \ref{fig:app-sd}, we present more results of subject-driven generation using both LoRA and \ourmethod.

\begin{figure*}[!h]
  \centering
  \includegraphics[width=1\linewidth]{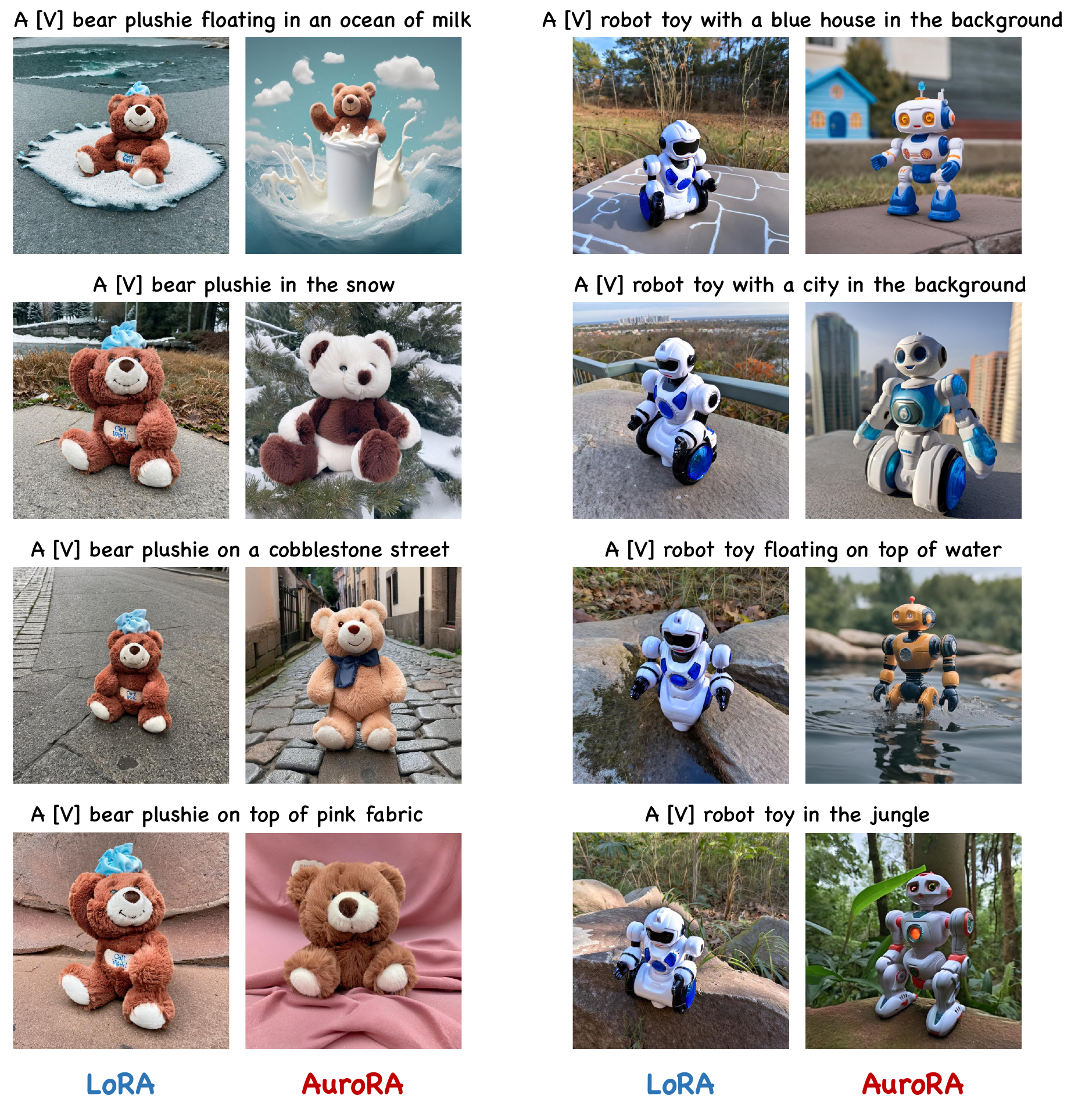}
  \caption{More generated images in the subject-driven generation task.}
   \label{fig:app-sd}
\end{figure*}

\end{document}